\documentclass{article} 
\usepackage[preprint]{icml2026}

\usepackage{pgfplots}
\pgfplotsset{compat=1.18}
\usepgfplotslibrary{fillbetween}
\usepackage{algorithm}
\usepackage{algorithmic}
\usepackage{amsthm, amssymb} 
\usepackage{amsmath}
\usepackage{soul}
\usepackage{ulem}
\usepackage{float} 
\usepackage{subcaption}
\usepackage{fancyhdr}
\usepackage{hyperref}
\usepackage{graphicx}
\usepackage{url}
\usepackage{arydshln}  
\usepackage{xcolor}
\usepackage{booktabs,multirow}
\usepackage[table]{xcolor} 
\newcommand{\best}[1]{\textbf{#1}}
\usepackage{enumitem}
\usepackage[
    font=small,             
    justification=raggedright, 
    singlelinecheck=true      
]{caption}

\newtheorem{theorem}{Theorem}
\newtheorem{corollary}{Corollary}
\newtheorem{remark}{Remark}

\begin{document}

\twocolumn[
  \icmltitle{Revisiting Adaptive Rounding with Vectorized Reparameterization for LLM Quantization}



  \icmlsetsymbol{equal}{*}

  \begin{icmlauthorlist}
    \icmlauthor{Yuli Zhou}{eth}
    \icmlauthor{Qingxuan Chen}{eth}
    \icmlauthor{Luca Benini}{eth,bologna}
    \icmlauthor{Guolei Sun}{nankai}
    \icmlauthor{Yawei Li}{eth}
  \end{icmlauthorlist}

  \icmlaffiliation{eth}{ETH Zurich, Zurich, Switzerland}
  \icmlaffiliation{nankai}{Nankai University, Tianjin, China}
  \icmlaffiliation{bologna}{University of Bologna, Bologna, Italy}

  \icmlcorrespondingauthor{Yawei Li}{li.yawei.ai@gmail.com}

  \icmlkeywords{Machine Learning, ICML}

  \vskip 0.3in
]

\printAffiliationsAndNotice{} 

\begin{abstract}
Adaptive Rounding has emerged as an alternative to round-to-nearest (RTN) for post-training quantization by enabling cross-element error cancellation. Yet, dense and element-wise rounding matrices are prohibitively expensive for billion-parameter large language models (LLMs). We revisit adaptive rounding from an efficiency perspective and propose \textbf{VQRound}, a parameter-efficient optimization framework that reparameterizes the rounding matrix into a compact codebook.
Unlike low-rank alternatives, VQRound minimizes the element-wise worst-case error under $L_\infty$ norm, which is critical for handling heavy-tailed weight distributions in LLMs. Beyond reparameterization, we identify rounding initialization as a decisive factor and develop a lightweight end-to-end finetuning pipeline that optimizes codebooks across all layers using only 128 samples.
Extensive experiments on OPT, LLaMA, LLaMA2, and Qwen3 models demonstrate that VQRound achieves better convergence than traditional adaptive rounding at the same number of steps while using as little as 0.2\% of the trainable parameters. Our results show that adaptive rounding can be made both scalable and fast-fitting. The code is available at \href{https://github.com/zhoustan/VQRound}{https://github.com/zhoustan/VQRound}.
\end{abstract}

\section{Introduction}
\begin{figure}[t] 
  \centering
  \includegraphics[width=0.9\linewidth]{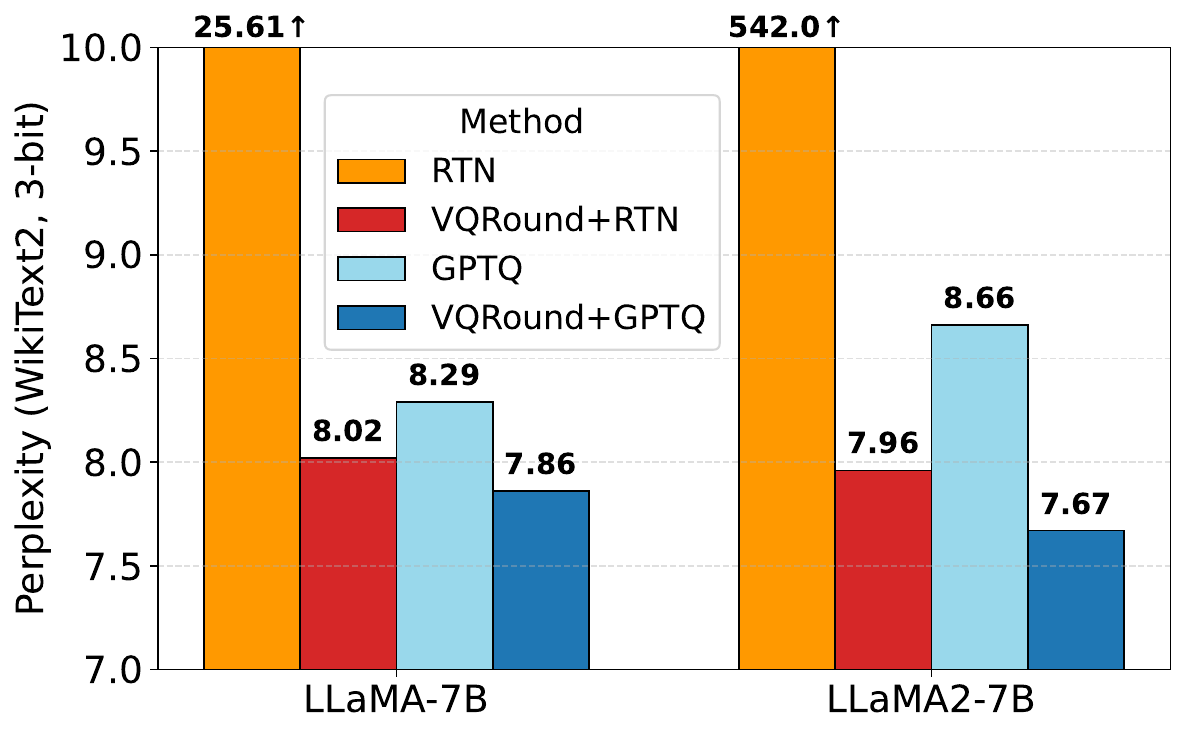} 
  \vspace{-2mm}
  \caption{WikiText-2 perplexity of LLaMA-7B and LLaMA2-7B under 3-bit quantization. VQRound delivers competitive performance and can be further combined with existing methods to reduce quantization error.}
  \captionsetup{justification=raggedright,singlelinecheck=false}
  \label{fig:teaser_1}
  \vspace{-6mm}
\end{figure}

Large language models (LLMs) have pushed the boundaries of natural language processing, but their rapid growth in parameters and context length comes with prohibitive compute and memory costs~\cite{kwon2023efficientmemorymanagementlarge}. Low-bit weight quantization has emerged as a crucial tool to reduce model size and inference latency while preserving task performance, enabling more scalable and efficient deployment of LLMs~\cite{dettmers2022llmint8}. While existing methods~\cite{frantar-gptq, quarot, egiazarian2024extreme, liu2024spinquant, tseng2024quip} predominantly rely on round-to-nearest (RTN) operations. This mathematical intuition-based approach is effective but theoretically suboptimal for minimizing global error. Nagel et al.\yrcite{adaround} point out that the task loss increase induced by quantization can be approximated by a quadratic form $\Delta{\mathbf{w}^\top}\cdot\mathbf{H}(\mathbf{w})\cdot\Delta{\mathbf{w}}$. When $\mathbf{H}$ has non-zero off-diagonal entries, the quantization error of different weights interacts via cross terms $\Delta{w_i}\Delta{w_j}$. This observation motivates adaptive rounding, which introduces learnable rounding parameters and optimizes rounding directions jointly to reduce global reconstruction or task loss. Several prior works apply adaptive rounding schemes~\cite{adaquant, adaround, kim2024aespa, flexround, lee2024lrqoptimizingposttrainingquantization} by introducing a rounding matrix that determines whether to round up or down from the nearest quantization bin for each weight entry.

Despite its theoretical advantages, adaptive rounding has so far seen limited success in large-scale LLM quantization. While effective on smaller networks (\textit{e.g.}, 4-bit CNN quantization in AdaRound~\cite{adaround}, these methods scale poorly to LLMs because the rounding matrix is the same size as the weight matrix. Existing methods introduce an auxiliary rounding matrix whose size matches the weight matrix. This not only inflates memory and computation overhead, but also creates a huge solution space that is difficult to optimize~\cite{OmniQuant,ding2025cbqcrossblockquantizationlarge}. In this work, we improve adaptive rounding from an efficiency perspective by proposing VQRound, a parameter-efficient adaptive rounding framework that rethinks how rounding strategies are represented and optimized. Instead of assigning an independent rounding variable to each weight, VQRound reparameterizes the rounding matrix in a smaller space using vector quantization (VQ). Specifically, rounding variables are represented via a compact codebook, compressing the optimization space from a full-size matrix to a small set of shared vectors. We also discover that VQRound can be potentially incorporated into existing post-training quantization pipelines according to experiments with GPTQ~\cite{frantar-gptq}, QuaRot~\cite{quarot}, and OmniQuant~\cite{OmniQuant}.

Beyond reparameterization, we identify rounding initialization as a critical yet previously underexplored factor in effective adaptive rounding. While prior works on adaptive rounding~\cite{adaround, adaquant, li2021brecq, flexround,lee2024lrqoptimizingposttrainingquantization} primarily focus on rounding objectives and optimization schedules, we show that the initial alignment between rounding parameters and full-precision residual structure plays a decisive role in optimization stability and convergence speed. Poor initialization can lead to unstable training dynamics and slow fitting, particularly in large models with highly coupled rounding variables. In contrast, a well-aligned initialization provides a strong starting point that significantly stabilizes optimization and accelerates convergence, enabling adaptive rounding to scale more effectively.

Previous adaptive rounding methods~\cite{adaround, adaquant, li2021brecq, flexround,lee2024lrqoptimizingposttrainingquantization} typically adopt block-wise (or layer-wise) reconstruction to optimize the rounding matrix. In this work, we further provide an optional end-to-end (E2E) fine-tuning strategy for VQRound to better exploit cross-layer interactions. This global optimization enables rounding errors to be compensated across layers, extending beyond purely local block-wise corrections. 

We develop a configurable optimization pipeline that supports both blockwise and end-to-end optimization, enabling explicit trade-offs between blockwise efficiency and end-to-end accuracy gains.

\begin{figure}[t] 
  \centering
  \includegraphics[width=1.0\linewidth]{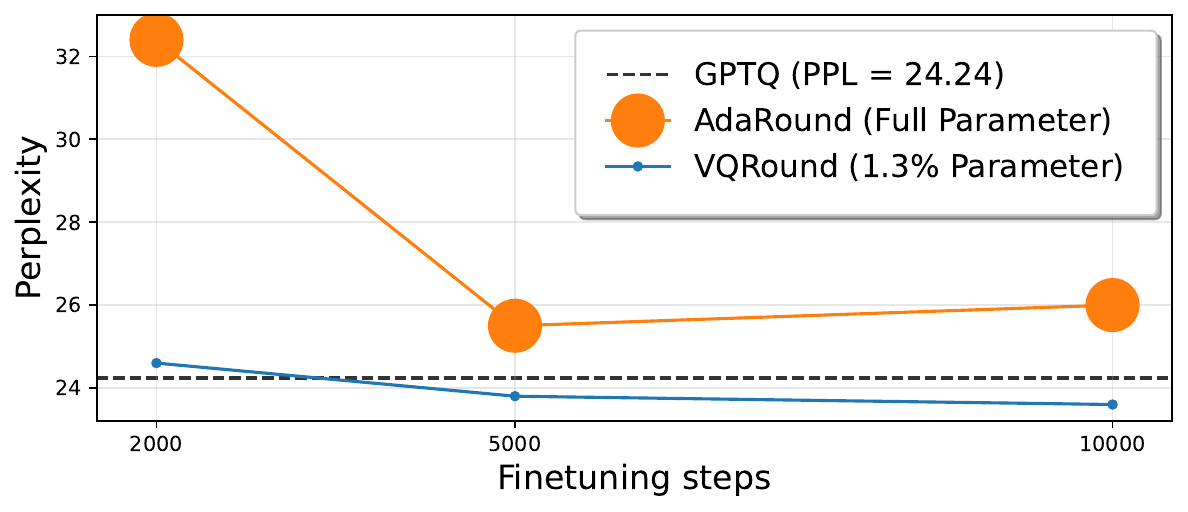} 
  \vspace{-4mm}

  \caption{Convergence comparison between AdaRound and VQRound on OPT-350M under the same end-to-end finetuning steps. The marker size indicates the number of trainable parameters. VQRound converges faster and achieves lower perplexity with significantly fewer trainable parameters.}
  \label{fig:teaser_2}
  \vspace{-8mm}
\end{figure}

In summary, the main contributions of this paper are:
\begin{itemize}[nosep,leftmargin=0.2in]
    \item We propose VQRound, a novel vectorized reparameterization of the adaptive rounding matrix that drastically reduces the number of learnable parameters to approximately 0.2\% of a billion-parameter model, while maintaining high quantization fidelity. 
    
    \item We theoretically analyze the error propagation property during reparameterization. This explains why vectorized reparameterization is superior to other methods (\textit{e.g.}, low-rank decomposition, Kronecker decomposition) and is support by experimental results.
    
    \item We identify rounding initialization as a critical factor in adaptive rounding and show that a well-aligned initialization significantly improves optimization stability and convergence.
    \item We develop an end-to-end fine-tuning pipeline as an alternative to traditional block-wise reconstruction, which jointly optimizes rounding codebooks under a global objective and enables cross-layer error compensation with minimal calibration data. 
\end{itemize}

\section{Related Work}
\subsection{Adaptive Rounding}
Rounding is the basic operation in almost all LLM quantization algorithms (In contrast, vector quantization addresses a sphere packing problem via K-means clustering or nearest-neighbor search~\cite{egiazarian2024extreme,tseng2024quip}. Thus, rounding is not needed.). Current quantization methods typically employ round-to-nearest (RTN) at different granularities, \textit{i.e.}, per-tensor, per-group, per-channel, or even element-wise. Despite its central role, few studies investigate how to optimize rounding itself for efficient and accurate LLM quantization. Recent work shows that stochastic rounding~\cite{croci2022stochastic} and adaptive rounding~\cite{gupta2015deep} can effectively cancel out quantization error, yielding better solutions than RTN.

In particular, adaptive rounding replaces rigid nearest-neighbor rounding with learnable or optimized rounding functions that minimize task-relevant reconstruction loss. Early works such as AdaRound~\cite{adaround} demonstrated that optimized rounding enables 4-bit quantization of CNNs, while subsequent methods like AdaQuant~\cite{adaquant} and BRECQ~\cite{li2021brecq} extended the framework with more flexible formulations, error metrics, and data-aware objectives. However, these do not show that adaptive rounding is ripe for scaling up in LLM quantization. Even though past efforts such as~\cite{flexround} and~\cite{lee2024lrqoptimizingposttrainingquantization} have made attempts to apply adaptive rounding in LLMs, the problems that will arise in scaling up are unavoidable. For example, ~\cite{OmniQuant} mentioned that the rounding matrix is hard to optimize in LLM due to its huge solution space. What's more, the reconstruction is usually block-by-block~\cite{wu2025fimaq}, which would bring potentially high costs on large models.

\subsection{Vector Quantization}
Beyond model quantization, vector quantization has been widely adopted as a discrete representation learning mechanism in generative modeling. VQ-VAE~\cite{vqvae} learns a discrete latent codebook that supports high-fidelity reconstruction and autoregressive priors, while VQ-GAN~\cite{esser2020taming} couples a learned codebook in order to synthesize high-resolution images. In embodied AI, UniAct~\cite{zheng2025universalactionsenhancedembodied} introduces a universal action space and action tokenizer that discretizes continuous robot controls into transferable action tokens. 

VQ has emerged as a powerful alternative to scalar quantization for compressing large models. Unlike scalar methods that treat weights independently, VQ learns a codebook of representative vectors and maps groups of parameters to code indices. Recent methods adapt VQ to LLMs. AQLM~\cite{egiazarian2024extreme} performs input-adaptive multi-codebook quantization. VPTQ~\cite{vptq} introduces a second-order optimization framework. QuIP\#~\cite{tseng2024quip} combines Hessian-aware compression with task-aware reconstruction. While effective at very low bit-widths, these approaches typically require expensive Hessian estimation or multi-codebook clustering, making them computationally and memory-intensive for PTQ. Different from the previous work, we use VQ to reparameterize the rounding matrix, which proves to be both efficient and accurate.

\section{Methodology}

\begin{figure*}[t] 
  \centering
  \includegraphics[width=\linewidth]{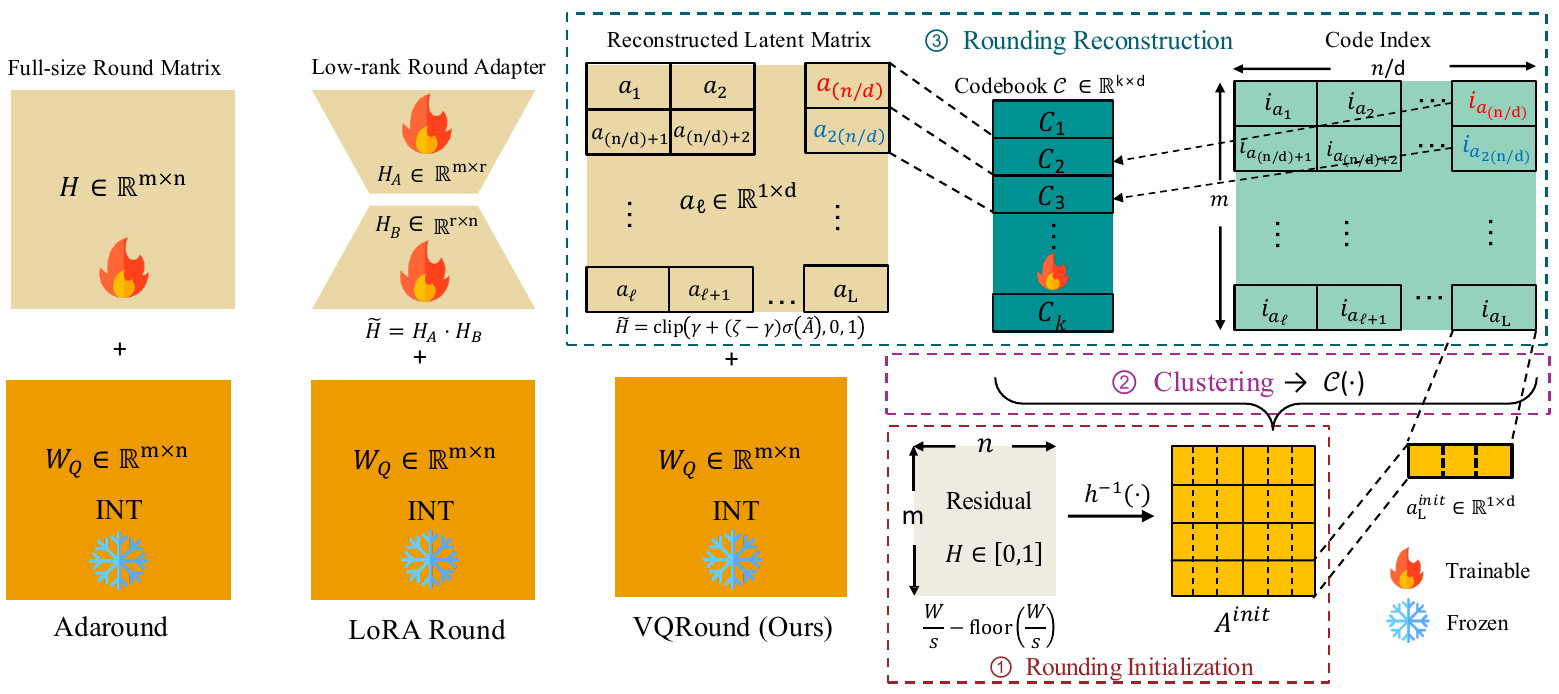} 
  \vspace{-6mm}
  \caption{Comparison of different rounding strategies
  including AdaRound~\cite{adaround}, LoRA Round~\cite{ding2025cbqcrossblockquantizationlarge} and VQRound. Rounding matrix is initialized by $W/s$ and its floored $floor(W/s)$, and $h^{-1}(\cdot)$ is the inverse rectified sigmoid transform. In VQRound, only the codebook $\mathcal{C}\in \mathbb{R}^{k\times d}$ needs to be updated. This is more parameter-efficient than AdaRound and LoRA Round.} 
  \vspace{-4mm}
  \label{fig:vq}
\end{figure*}

In this section, we introduce the details of VQRound, which dramatically improves the efficiency of adaptive rounding. As shown in Fig.~\ref{fig:vq}, the core of the method is the rounding matrix, initialized as the error of the floor operation. A vector codebook is utilized to reparameterize the rounding matrix. In \S\ref{sec:err_prop}, we discuss the rationale for this approach. In \S\ref{sec:vq} we introduce the vectorized reparameterization method and its optimality over other reparameterization methods. Furthermore, we identify rounding initialization as a critical yet underexplored factor in \S\ref{sec:init}. To cope with the suboptimal localized optimization of traditional adaptive rounding methods, we introduce an end-to-end finetuning approach that enables globally optimal quantization in a one-time finetuning stage in \S\ref{sec:blockwise_e2e}.

\textbf{Notation.} We denote the weight matrix by $W$, its quantized form by $W_Q$ (or simply $Q$), the $j$-th column of $W$ by $W_{j}$, and the submatrix consisting of columns from the $j$-th onward (inclusive) by $W_{j:}$. Matrices are denoted by uppercase letters, while scalars are represented by lowercase letters. The Hessian matrix is denoted by $\mathbf{H}$ and the rounding matrix is represented by $H$ throughout this paper. Let $\tilde{H}$ be its approximation (\textit{e.g.}, quantized rounding matrix), and $\Delta H = H - \tilde{H}$ is the error.

\subsection{Preliminaries}
\label{subsec:preliminary}
We follow standard uniform affine weight quantization. For a weight matrix $W\!\in\!\mathbb{R}^{m\times n}$ and $b$-bit integers, the quantized integer tensor $W_Q$ under RTN is
\vspace{-1mm}
\begin{equation}
W_Q = \text{clip}(\text{round}(\frac{W}{s})+z, q_{min}, q_{max}),
\end{equation}
where $s$ is the 
scale, $z$ is the zero-point,
and $q_{\min},q_{\max}$ are the valid integer range. 
RTN independently rounds each entry of $W/s$ to its nearest integer. Though intuitive, this direct conversion may introduce quantization error due to the ignorance of cross-term interactions~\cite{adaround}. To solve this problem, adaptive rounding replaces the rigid RTN decision with a rounding matrix $H\!\in\![0,1]^{m\times n}$ that controls the up/down rounding of entries:
\begin{align}
    W_Q = \mathrm{clip}\!\left(\Big\lfloor \frac{W}{s} \Big\rfloor + H + z,\; q_{\min},\, q_{\max}\right), \label{eq:adaround} \\
    H_{ij} = h(A_{ij}) = \mathrm{clip}\ \!\Big(\gamma+(\zeta-\gamma)\sigma(A_{ij}),\,0,\,1\Big), \label{eq:rounding_matrix}
\end{align}

where the latent rounding matrix $A$ contains the learnable parameter$A_{ij}$, 
$h(\cdot)$ is a rectified sigmoid proposed in~\cite{louizos2017learning},
$\sigma(\cdot)$ is the sigmoid function, $\gamma=-0.1$, and $\zeta=1.1$. During optimization, $H_{ij}$ gradually annealed toward binary values $\{0,1\}$, and $H_{ij}=0$ means rounding-down, 
while $H_{ij}=1$ means rounding-up.
During initialization of the parameters,
$H_{ij}$ is set as the residual between the original weights and their floor values, \textit{i.e.}, $W/s - \lfloor {W}/{s} \rfloor$, which is constrained to the interval $[0,1]$.

At the inference stage, the resulting binary rounding matrix provides deterministic rounding decisions.

The problem of the existing adaptive rounding scheme is that the rounding matrix has the same dimension as the original weights, which results in a large optimization space and cannot be scaled to huge models with billions of parameters. This is the reason why AdaRound could not be applied to LLM quantization~\cite{adaround}. In the following, we introduce a method that shrinks the degree of freedom in $A$ with some parameter reduction method.

\subsection{Analysis of Error Propagation for $h$}
\label{sec:err_prop}

The inverse process of parameter reduction is the construction of the latent rounding matrix $A$, which can be abstracted as a generating function driven by the parameter set $\Theta$, \textit{i.e.}, $\tilde{A} = f(\Theta)$.
Then by applying the nonlinear transformation in Eq.~\ref{eq:rounding_matrix}, the rounding matrix used in quantization is derived.

In this paper, we seek a function $f(\cdot)$ that reduces the number of learnable parameters, \textit{i.e.}, $|\Theta| \ll mn$, while preserving sufficient expressive power in $f(\Theta) \in \mathbb{R}^{m \times n}$ to effectively minimize the approximation error of the rounding matrix, \textit{i.e.},  $\Delta h({A})  = h({A}) - {h}(\tilde{A})$. To achieve this, we analyze how errors induced by parameter reduction propagate to the rounding matrix by investigating the property of the transformation $h$.

\begin{theorem}[Lipschitz contraction of $h$]
\label{thm:lipschitz_contraction}
The transformation function $h: A \mapsto H$ is globally Lipschitz with constant $\mathcal{L} = ({\zeta-\gamma})/{4}$. For any $A,\tilde{A}\in\mathbb{R}^{m\times n}$ and $H_{ij}=h(A_{ij})$, $\tilde H_{ij}=h(\tilde{A}_{ij})$,
\vspace{-1mm}
\begin{equation}
\label{eq:elemwise_lipschitz_main}
|\,\tilde H_{ij}-H_{ij}\,| \le 
\mathcal{L}\,|\,\tilde{A}_{ij}-A_{ij}\,|,
\qquad \forall (i,j).
\end{equation}
\vspace{-1mm}
In particular,
\vspace{-1mm}
\begin{equation}
\label{eq:maxnorm_lipschitz_main}
\|\tilde H-H\|_{\infty} \le \mathcal{L}\,\|\tilde{A}-A\|_{\infty}.
\end{equation}
\end{theorem}
\begin{corollary}[Tail transfer under Lipschitz contraction]
\label{cor:lipschitz_tail_transfer}
Under the assumptions of Theorem~\ref{thm:lipschitz_contraction}, let
$\Delta{A}_{ij}=\tilde{A}_{ij}-A_{ij}$ and
$\Delta H_{ij}=\tilde H_{ij}-H_{ij}$. Then for any $\varepsilon>0$,
\begin{equation}
\label{eq:lipschitz_tail_transfer}
\mathbb{P}\big(|\Delta H_{ij}|>\varepsilon\big)
\;\le\;
\mathbb{P}\Big(|\Delta{A}_{ij}|>\varepsilon/\mathcal{L}\Big).
\end{equation}
\end{corollary}

\begin{remark}
Corollary~\ref{cor:lipschitz_tail_transfer} shows that the nonlinear transformation $h$ in Eq.~\ref{eq:rounding_matrix} cannot amplify small or moderate element-wise errors in $A$; reductions in
the tail probability of $|\Delta{A}_{ij}|$ directly translate to reductions
in the tail probability of $|\Delta H_{ij}|$.
\end{remark}

\begin{theorem}[Clipping threshold via margin-to-boundary]
\label{thm:clipping_threshold}
Let $h$ be defined as in Eq.~\ref{eq:rounding_matrix} and define
$g(x)=\gamma+(\zeta-\gamma)\sigma(x)$. Fix $(i,j)$ and assume
$g(A_{ij})\in(0,1)$.
Define the margin to the nearest clipping boundary as
\vspace{-1mm}
\begin{equation}
\label{eq:margin_def_threshold}
\delta_{ij}
= \min\big\{g(A_{ij}),\,1-g(A_{ij})\big\}.
\end{equation}
If $|\tilde{A}_{ij}-A_{ij}|>\delta_{ij}/\mathcal{L}$ ($\mathcal{L}$ is defined in Theorem~\ref{thm:lipschitz_contraction}), then clipping is activated at entry $(i,j)$, \textit{i.e.},
\begin{equation}
\tilde H_{ij}\in\{0,1\}.
\end{equation}
\end{theorem}
\vspace{-1mm}
\begin{corollary}[Probability of clipping-induced saturation]
\label{cor:clipping_probability}
Under the assumptions of Theorem~\ref{thm:clipping_threshold},
\vspace{-1mm}
\begin{equation}
\label{eq:clipping_probability_bound}
\mathbb{P}\big(\tilde H_{ij}\in\{0,1\}\big) \le
\mathbb{P}\Big(|\Delta{A}_{ij}|>\delta_{ij}/\mathcal{L}\Big).
\end{equation}
\end{corollary}
\vspace{-1mm}
\begin{remark}
Corollary~\ref{cor:clipping_probability} isolates a nonlinear effect absent from
pure Lipschitz analysis: large element-wise errors in $A$ can induce
clipping and hence produce order-one errors in $H$. As a result, approximation
methods that suppress the tails of $|\Delta{A}_{ij}|$ reduce the frequency of
clipping-induced saturation errors.
\end{remark}
\vspace{-2mm}
All proofs are given in Appendix~\ref{sec:proof}.

\begin{figure*}[t] 
  \centering
  \includegraphics[width=0.95\linewidth]{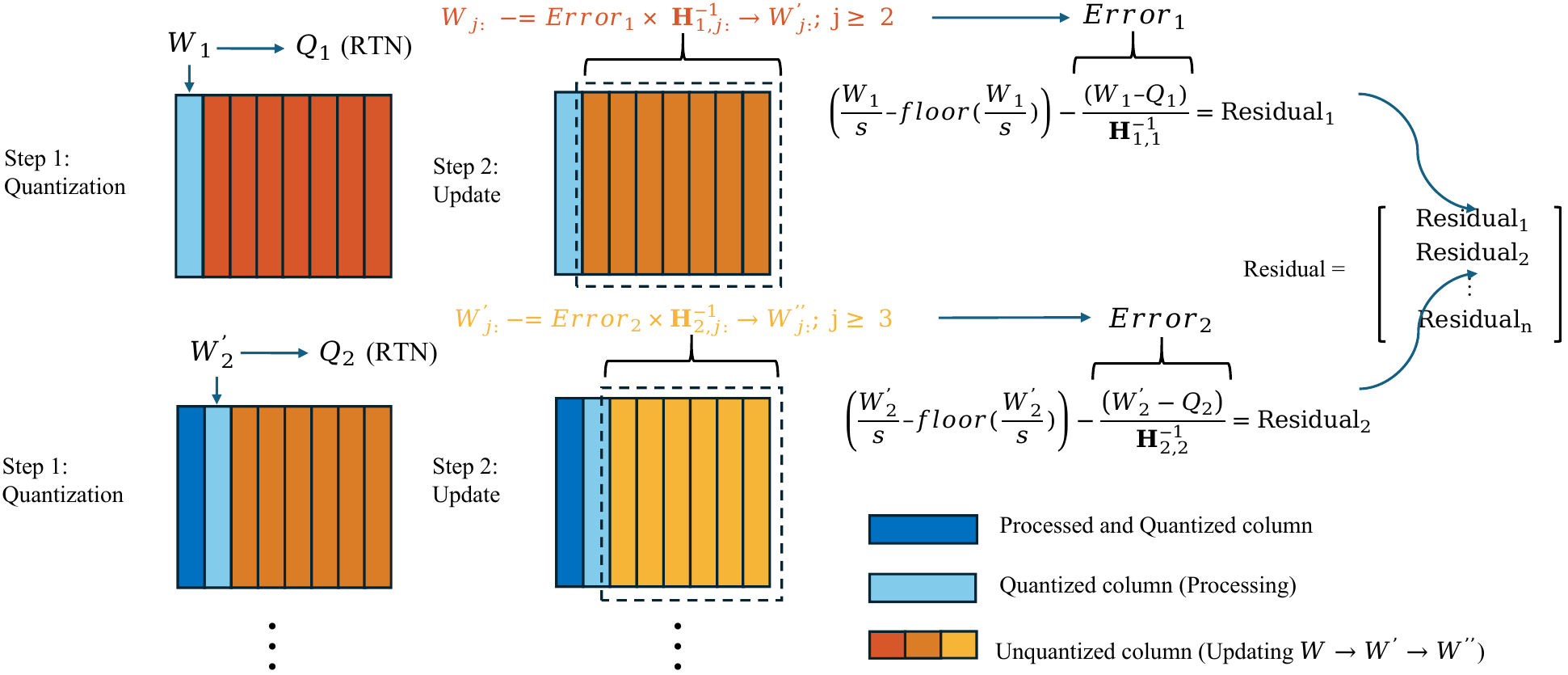} 
  \caption{Initialization of rounding matrix based on residual before inverse rectified sigmoid.}

  \vspace{-2mm}
  \label{fig:initialization}
\end{figure*}

\subsection{Vectorized Reparameterization}
\label{sec:vq}

Several methods can be employed for parameter reduction in adaptive rounding. For example, CBQ~\cite{ding2025cbqcrossblockquantizationlarge} adopts singular value decomposition (SVD), and Kronecker product–based decompositions can be used for a similar purpose. However, in our exploration, we find that these approaches are not well-suited for suppressing the tails of the element-wise error distribution  $|\Delta A_{ij}|$. SVD and Kronecker-based approximations impose global low-dimensional structure on $A$ and optimize global reconstruction objectives (\textit{e.g.}, Frobenius norm). While effective in reducing average error, these global approximations do not explicitly control element-wise deviations and may produce localized residual spikes. Such spikes manifest as heavier element-wise tails in $|\Delta A_{ij}|$, increasing the probability that the clipping threshold
$\delta_{ij}/\mathcal{L}$ is crossed. Each threshold crossing triggers saturation in the rounding function, leading to order-one errors in $H$ and a broader error distribution. More analysis is given in Appendix~\ref{sec:svd_kronecker}~\&~\ref{sec:figure}.

To explicitly minimize the worst-case error, we propose to reparameterize $A$ using vector quantization (VQ). We divide $A$ into $L$ vectors of length $d$,
$\{a_\ell\in\mathbb{R}^{d}\}_{\ell=1}^{L}$, and learn a codebook
$\mathcal{C}=\{c_1,\dots,c_{k}\}\subset\mathbb{R}^{d}$ with $k\ll L$.
Each vector is assigned to its nearest centroid:
\begin{equation}
i_\ell=\arg\min_{j\in \{1,\dots,k\}}\;\|a_\ell-c_j\|_2^2,\qquad
a^{\mathrm{VQ}}_\ell=c_{i_\ell},
\end{equation}
yielding the VQ-based reconstruction of the rounding matrix
$A_{\mathrm{VQ}}=\mathrm{reshape}\big([\,c_{i_1},\dots,c_{i_L}\,]\big)$.
In forward passes, the substitution is a table lookup; only the codebook vectors
$\{c_1, \cdots, c_k\}$ are trainable.  
This reduces the number of learned parameters from $\mathcal{O}(mn)$ to $\mathcal{O}(kd)$, while preserving local flexibility within each block.

Vector quantization produces locally bounded reconstruction errors by assigning each block of $A$ to a nearby codeword. As a result, element-wise errors are constrained by the local cell geometry and large deviations are suppressed. In probabilistic terms, vector quantization yields a rapidly decaying tail and rarely produces errors large enough to exceed the clipping threshold $\delta_{ij}/\mathcal{L}$. Consequently, clipping events are infrequent, and the induced element-wise errors in $H$
remain tightly concentrated around zero.

\subsection{Rounding Initialization}
\label{sec:init}
We find that the initialization of the rounding matrix is decisive for optimization: with identical training schedules, different initializations can lead to remarkably different final results. Since rounding parameters determine whether each weight rounds up or down, a poor starting point can induce unstable optimization, slow convergence, and suboptimal rounding patterns that accumulate large quantization error. An effective initialization should (i) reduce inference error from the outset and (ii) explicitly control the worst-case element-wise error, \textit{i.e.}, $\|\tilde H-H\|_{\infty}$, so that each entry is already oriented toward a correct descent direction.

\textbf{Hessian-aware residual initialization.} 

To construct a more informed initialization, we propose a Hessian-aware residual scheme. We describe the procedure for channel-wise quantization with sequential error compensation, as illustrated in Fig.~\ref{fig:initialization}. Followed by GPTQ~\cite{frantar-gptq}, we process the weight column by column. At step $j$, we first quantize the $j$-{th} column using RTN to obtain $Q_j$,  and the quantization error $E_j$ is defined as:
\vspace{-1mm}
\begin{equation}
    E_j = (W_j - Q_j)/\mathbf{H}_{j,j}^{-1},
\end{equation}
which normalizes the local quantization error by the curvature along column $j$. Instead of discarding this residual, we propagate it to the remaining unprocessed columns using Hessian coupling terms:
\vspace{-1mm}
\begin{equation}
    W_{j:} \leftarrow W_{j:} - E_j \cdot \mathbf{H}_{j,j:}^{-1},
\end{equation}
This sequential update accumulates second-order compensation effects from previously processed columns. Finally, we use $E_j$ to initialize the approximate rounding variable of column $j$ as:
\vspace{-2mm}
\begin{equation}
    \tilde{H_j} =  \frac{W_j}{s} - \Big\lfloor \frac{W_j}{s} \Big\rfloor - E_j.
\end{equation}
\vspace{-1mm}
The full procedure is provided in Appendix Alg.~\ref{alg:hessian_init}.

\subsection{Blockwise or End-to-End Optimization}
\label{sec:blockwise_e2e}

While vectorized reparameterization provides a compact initialization of the rounding matrix $H$, the rounding parameters can be further optimized after initialization. In practice, VQRound supports both blockwise reconstruction and end-to-end finetuning. Either strategy can be adopted depending on computational constraints and optimization objectives. Blockwise reconstruction optimizes rounding decisions locally within each network block~\cite{li2021brecq}, which reduces memory overhead and simplifies optimization. In contrast, end-to-end (E2E) finetuning jointly updates the codebook parameters across all layers under a global objective. This global optimization allows cross-layer error compensation and may improve reconstruction fidelity when the computing resource is available. The E2E finetuning procedure is summarized in Alg.~\ref{alg:finetune}. For both methods, we optimize a loss consisting of a KL-divergence~\cite{kldiv} and a regularization term $\mathcal{R}(H)$ that encourages the relaxed rounding variables to approach hard $\{0,1\}$ decisions, \textit{i.e.},

\begin{equation}
    \mathcal{R}(H)=\sum_{i,j}[1-|2H_{ij}-1|^\beta].
\end{equation}
The parameter $\beta$ is annealed during training to gradually sharpen the rounding matrix. A large $\beta$ may slow convergence of the rounding objective, whereas a small $\beta$ may produce overly sharp rounding decisions early in training.

\section{Experiments}

We evaluate VQRound on a diverse set of language model families, including OPT~\citep{optmodel}, LLaMA~\citep{llama}, LLaMA2~\citep{touvron2023llama2openfoundation}, which are widely adopted in both large language model applications and quantization research~\citep{frantar-gptq, egiazarian2024extreme}. To further examine its robustness, we additionally report results on the recent Qwen3 model~\citep{yang2025qwen3technicalreport}. We also validate the compatibility of VQRound by integrating it with existing quantization frameworks such as GPTQ~\citep{frantar-gptq}, QuaRot~\citep{quarot}, and OmniQuant~\cite{OmniQuant}. Detailed experimental settings are provided in Appendix~\ref{appendix:exp_setting}.

\subsection{Results}
We evaluate the models using perplexity on WikiText2~\citep{wikitext2} and C4~\citep{c4dataset}. In addition, we report zero-shot accuracy on WinoGrande~\citep{windograde}, PiQA~\citep{piqa}, HellaSwag~\citep{hellaswag}, and ARC-Easy/ARC-Challenge~\citep{arcchallange}. Additional experimental results are provided in Appendix~\ref{appendix:results}.

\begin{table}[h]
\centering
\caption{OPT perplexity on Wikitext2. Lower is better. *: We integrate VQRound with OmniQuant~\cite{OmniQuant}, following the design of block-wise reconstruction. All the results from OmniQuant were obtained through our retraining process, and the initialization of VQ was also established during this training process.}
\label{tab:opt-wiki}
\vspace{-3mm}
\setlength{\tabcolsep}{1.5pt}
\scalebox{0.69}{
    \begin{tabular}{cc cccccc}
    \toprule
    \multirow{2}{*}{\textbf{Precision~~}} & \multirow{2}{*}{\textbf{~~Method~~}}
    & \multicolumn{6}{c}{\textbf{OPT Model Size}} \\
    \cmidrule(lr){3-8}
    & & \textbf{~~~125M} & \textbf{~~~350M} & \textbf{~~1.3B~} & \textbf{~~2.7B} & \textbf{~~6.7B} & \textbf{~~13B~} \\
    \midrule
    \multirow{1}{*}{FP16}
    & Baseline
    & 27.65 & 22.00 & 14.63 & 12.47 & 10.86 & 10.13 \\
    \specialrule{1pt}{3pt}{3pt}
    \multirow{6}{*}{\cellcolor{white}4 bits}
    & RTN                   & 37.29 & 25.94 & 48.17 & 16.92 & 12.10 & 11.32 \\
    & \cellcolor{gray!30} VQRound+RTN & \cellcolor{gray!30}30.69 & \cellcolor{gray!30}23.77 & \cellcolor{gray!30}15.48 & \cellcolor{gray!30}13.30 & \cellcolor{gray!30}11.26 & \cellcolor{gray!30}10.66 \\
    & GPTQ                  & 31.12 & 24.24 & 15.47 & 12.87 & 11.39 & \textbf{10.31} \\
    & \cellcolor{gray!30} VQRound+{GPTQ}  & \cellcolor{gray!30}\best{30.39} & \cellcolor{gray!30}\best{23.02} & \cellcolor{gray!30}\best{15.38} & \cellcolor{gray!30}\best{12.77} & \cellcolor{gray!30}\best{11.13} & \cellcolor{gray!30}10.37 \\ 
    \cmidrule(lr){2-8}
    & OmniQuant$^{*}$ & 30.40 & - & 15.61 & 13.34 & 11.58 & 10.89 \\
    & \cellcolor{gray!30} VQRound+{OmniQuant}$^{*}$  & \cellcolor{gray!30}\best{28.98} & \cellcolor{gray!30}\best{-} & \cellcolor{gray!30}\best{15.19} & \cellcolor{gray!30}\best{12.86} & \cellcolor{gray!30}\best{11.12} & \cellcolor{gray!30}\best{10.44} \\
    
    \specialrule{1pt}{3pt}{3pt}
    
    \multirow{6}{*}{3 bits}
    & RTN                              & 1.3e3 & 64.57 & 1.3e4 & 1.6e4 & 5.8e3 & 3.4e3 \\
    &  \cellcolor{gray!30} VQRound+RTN  & \cellcolor{gray!30}47.02 & \cellcolor{gray!30}33.63 & \cellcolor{gray!30}22.67 & \cellcolor{gray!30}18.57 & \cellcolor{gray!30}13.72 & \cellcolor{gray!30}12.28 \\
    & GPTQ                             & 53.85 & 33.79 & 20.97 & 16.88 & 14.86 & 11.61 \\
    &  \cellcolor{gray!30} VQRound+GPTQ & \cellcolor{gray!30}\best{46.10} & \cellcolor{gray!30}\best{28.03} & \cellcolor{gray!30}\best{19.13} & \cellcolor{gray!30}\best{15.55} & \cellcolor{gray!30}\best{12.45} & \cellcolor{gray!30}\best{11.37} \\
    \cmidrule(lr){2-8}
    & OmniQuant$^{*}$ & 41.61 & - & 19.03 & 16.69 & 13.32 & 13.59 \\
    & \cellcolor{gray!30} VQRound+{OmniQuant}$^{*}$  & \cellcolor{gray!30}\best{32.92} & \cellcolor{gray!30}\best{-} & \cellcolor{gray!30}\best{19.00} & \cellcolor{gray!30}\best{13.99} & \cellcolor{gray!30}\best{11.93} & \cellcolor{gray!30}\best{11.23} \\

    \bottomrule
    \end{tabular}%
}
\end{table}

\begin{table}[h]
\centering
\small
\caption{Applying VQRound on QuaRot~\citep{quarot} improves Wikitext2 PPL under W4A16 asymmetric quantization.}
\label{tab:quarot-vq}
\setlength{\tabcolsep}{6pt}
\renewcommand{\arraystretch}{1.1}

\begin{tabular}{lcc}
\toprule
\textbf{Method} & \textbf{LLaMA 7B} & \textbf{LLaMA2 7B} \\
\midrule
FP16 & 5.68 & 5.47 \\
\midrule
RTN & 7.94 & 6.99 \\
QuaRot+RTN & 7.46 & 6.76 \\
QuaRot+VQRound & \textbf{5.98} & \textbf{5.84} \\
\bottomrule
\end{tabular}
\vspace{-2mm}
\end{table}

\begin{table*}[!t]
    \centering
    \caption{LLaMA family perplexity on Wikitext2 and C4. Lower is better.}
    \label{tab:llama-family-wiki-c4}
    \renewcommand{\arraystretch}{1.2} 
    \resizebox{0.8\linewidth}{!}{%
    \begin{tabular}{cc cccc cccc}
    \toprule
    \multirow{2}{*}{\textbf{Precision}} & \multirow{2}{*}{\textbf{Method}}
      & \multicolumn{2}{c}{\textbf{LLaMA-7B}} 
      & \multicolumn{2}{c}{\textbf{LLaMA-13B}} 
      & \multicolumn{2}{c}{\textbf{LLaMA2-7B}} 
      & \multicolumn{2}{c}{\textbf{LLaMA2-13B}} \\
    \cmidrule(lr){3-4} \cmidrule(lr){5-6} \cmidrule(lr){7-8} \cmidrule(lr){9-10}
    & & \textbf{WikiText2} & \textbf{C4} 
      & \textbf{WikiText2} & \textbf{C4} 
      & \textbf{WikiText2} & \textbf{C4} 
      & \textbf{WikiText2} & \textbf{C4} \\
    \midrule
    FP16 & Baseline & 5.68 & 7.34 & 5.09 & 6.80 & 5.47 & 7.26 & 4.88 & 6.73 \\
    \specialrule{1pt}{3pt}{3pt}
    \multirow{4}{*}{\centering 4 bits}
      & RTN & 6.29 & 8.12 & 5.53 & 7.23 & 6.12 & 8.17 & 5.21 & 7.14 \\
      & \cellcolor{gray!30} VQRound+RTN  & \cellcolor{gray!30}6.13 & \cellcolor{gray!30}7.88 & \cellcolor{gray!30}5.42 & \cellcolor{gray!30}7.17 & \cellcolor{gray!30}5.90 & \cellcolor{gray!30}7.88 & \cellcolor{gray!30}5.19 & \cellcolor{gray!30}7.13 \\
      & GPTQ & 6.17 & 7.80 & \textbf{5.37} & 7.28 & 6.06 & 7.84 & \textbf{5.16} & \textbf{7.03} \\
      & \cellcolor{gray!30} VQRound+GPTQ & \cellcolor{gray!30}\textbf{6.08} & \cellcolor{gray!30}\textbf{7.78} & \cellcolor{gray!30}5.40 & \cellcolor{gray!30}\textbf{7.10} & \cellcolor{gray!30}\textbf{5.85} & \cellcolor{gray!30}\textbf{7.79} & \cellcolor{gray!30}5.18 & \cellcolor{gray!30}7.06 \\
    \specialrule{1pt}{3pt}{3pt}
    \multirow{4}{*}{3 bits}
      & RTN & 25.61 & 30.86 & 11.78 & 14.46 & 542.0 & 527.2 & 10.69 & 13.87 \\
      & \cellcolor{gray!30} VQRound+RTN & \cellcolor{gray!30}8.02 & \cellcolor{gray!30}10.29 & \cellcolor{gray!30}6.71 & \cellcolor{gray!30}8.88 & \cellcolor{gray!30}7.96 & \cellcolor{gray!30}10.54 & \cellcolor{gray!30}6.58 & \cellcolor{gray!30}8.94 \\
      & GPTQ       & 8.29 & 10.51 & 6.73 & 8.83 & 8.66 & 11.24 & 6.55 & 8.76 \\
      & \cellcolor{gray!30}VQRound+GPTQ & \cellcolor{gray!30}\textbf{7.86} & \cellcolor{gray!30}\textbf{9.95} & \cellcolor{gray!30}\textbf{6.46} & \cellcolor{gray!30}\textbf{8.47} & \cellcolor{gray!30}\textbf{7.67} & \cellcolor{gray!30}\textbf{10.05} & \cellcolor{gray!30}\textbf{6.33} & \cellcolor{gray!30}\textbf{8.57}   \\
    \specialrule{1pt}{3pt}{3pt}
    \multirow{4}{*}{2 bits}
      & RTN & 1.1e5 & 1.1e5 & 5.7e4 & 5.9e4 & 1.8e4 & 5.1e4 & 2.8e4 & 5.3e4 \\
      & \cellcolor{gray!30} VQRound+RTN & \cellcolor{gray!30} 65.41 & \cellcolor{gray!30}43.52 & \cellcolor{gray!30}47.57 & \cellcolor{gray!30}31.53 & \cellcolor{gray!30}84.07 & \cellcolor{gray!30}56.67 & \cellcolor{gray!30}68.27 & \cellcolor{gray!30}38.54 \\
      & GPTQ & 1.0e4 & 872.7 & 3.7e3 & 809.7 & 7.5e3 & 1.7e3 & 2.1e3 & 560.7 \\
      & \cellcolor{gray!30}VQRound+GPTQ & \cellcolor{gray!30}\textbf{64.82} & \cellcolor{gray!30}\textbf{37.49} & \cellcolor{gray!30}\textbf{34.62} & \cellcolor{gray!30}\textbf{25.20} & \cellcolor{gray!30}\textbf{73.08} & \cellcolor{gray!30}\textbf{45.13} & \cellcolor{gray!30}\textbf{48.29} & \cellcolor{gray!30}\textbf{29.68} \\
    \bottomrule
    \end{tabular}}

\end{table*}

\begin{table*}[!t]
  \centering
  \caption{Qwen3 perplexity on Wikitext2 and C4. Lower is better.}
  \label{tab:qwen-wiki-c4}
  \footnotesize
  \renewcommand{\arraystretch}{1.12}
  \setlength{\tabcolsep}{6pt}

\resizebox{0.8\linewidth}{!}{%
\begin{tabular}{cc*{8}{c}}
\toprule
\multirow{2}{*}{\textbf{Precision}} & \multirow{2}{*}{\textbf{Method}}
& \multicolumn{2}{c}{\textbf{0.6B}} & \multicolumn{2}{c}{\textbf{1.7B}}
& \multicolumn{2}{c}{\textbf{4B}}   & \multicolumn{2}{c}{\textbf{8B}} \\
\cmidrule(lr){3-4}\cmidrule(lr){5-6}\cmidrule(lr){7-8}\cmidrule(lr){9-10}
& & \textbf{WikiText2} & \textbf{C4} & \textbf{WikiText2} & \textbf{C4}
  & \textbf{WikiText2} & \textbf{C4} & \textbf{WikiText2} & \textbf{C4} \\
\midrule

\multirow{1}{*}{FP16}
& Baseline & 20.96 & 30.31 & 16.67 & 22.36 & 13.64 & 19.83 & 9.72 & 15.42 \\
\specialrule{1pt}{3pt}{3pt}

\multirow{4}{*}{\cellcolor{white}4 bits}
& RTN              & 37.39 & 51.69 & 28.26 & 32.45 & 17.47 & 24.57 & 12.01 & 18.48 \\
& \cellcolor{gray!30}VQRound+RTN  & \cellcolor{gray!30}25.55 & \cellcolor{gray!30}35.30 & \cellcolor{gray!30}\best{16.97} & \cellcolor{gray!30}25.08 & \cellcolor{gray!30}\best{13.57} & \cellcolor{gray!30}21.59 & \cellcolor{gray!30}10.33 & \cellcolor{gray!30}16.76 \\
& GPTQ             & 30.05 & 42.71 & 25.60 & 30.68 & 14.82 & \best{20.88} & 10.59 & 16.43 \\
& \cellcolor{gray!30}VQRound+GPTQ & \cellcolor{gray!30}\best{24.72} & \cellcolor{gray!30}\best{34.28} & \cellcolor{gray!30}17.00 & \cellcolor{gray!30}\best{24.15}
                    & \cellcolor{gray!30}13.73 & \cellcolor{gray!30}20.93 & \cellcolor{gray!30}\textbf{10.18} & \cellcolor{gray!30}\textbf{16.28} \\

\bottomrule
\end{tabular}
}
\vspace{-2mm}
\end{table*}

Tab.~\ref{tab:opt-wiki} reports perplexity results on WikiText2 across OPT model sizes and quantization settings. As a standalone rounding strategy, VQRound achieves performance comparable to existing SOTA approaches such as GPTQ and OmniQuant under both 4-bit and 3-bit quantization. In several 4-bit settings, particularly for smaller and mid-sized models, VQRound slightly improves over GPTQ while remaining competitive with OmniQuant.
In addition, VQRound has the potential to be integrated with existing post-training quantization methods. When combined with GPTQ or OmniQuant, it further reduces perplexity in most cases, with more noticeable gains in lower-bit settings. These results suggest that VQRound provides complementary error reduction on top of existing optimization procedures.
More results on the C4 (Tab.~\ref{tab:opt-c4}) exhibit the same characteristics as WikiText2. Both results show that VQRound has good compensation of RTN, GPTQ, and OmniQuant. Experiments on QuaRot (Tab.~\ref{tab:quarot-vq}) further show that VQRound can be combined with rotation-based quantization methods and provides additional performance gains.

We further evaluate VQRound on the LLaMA and LLaMA2 families (Tab.~\ref{tab:llama-family-wiki-c4}). At 4-bit precision, it achieves results on par with GPTQ, while at 3-bit VQRound+GPTQ consistently outperforms GPTQ (\textit{e.g.}, reducing perplexity by 0.43 on LLaMA-7B and 0.99 on LLaMA2-7B). Under the extreme 2-bit setting where GPTQ collapses, VQRound remains stable with perplexity below 100, demonstrating its robustness and potential for ultra-low-bit quantization.

As shown in Tab.~\ref{tab:qwen-wiki-c4}, VQRound demonstrates clear advantages on the Qwen3 family. It consistently mitigates the degradation of RTN and complements GPTQ, yielding stable improvements across model scales, with only negligible gaps in rare cases (\textit{e.g.}, C4 on Qwen3-4B). This confirms that VQRound generalizes effectively to modern LLMs.

We report zero-shot evaluation on five commonsense reasoning benchmarks in Tab.~\ref{tab:3bit-llama-commonsense} and Tab.~\ref{tab:4bit-llama-commonsense}. 
At 3-bit precision, VQRound+GPTQ achieves the highest accuracy in the majority of cases, especially on LLaMA2 models, where it consistently surpasses GPTQ. At 4-bit precision, our method shows comparable results with GPTQ. Notably, the advantage of our method becomes more evident under lower-bit settings, where it consistently reduces the accuracy gap to FP16 compared with GPTQ.

Furthermore, our comparative experiments on the OPT-350M model (in Fig.~\ref{fig:teaser_2}) demonstrate that VQRound exhibits faster convergence and lower perplexity than AdaRound using the same E2E fine-tuning steps.

\begin{table*}[t]
  \centering
  \footnotesize
  \caption{3 bit zero-shot accuracy (\%) on commonsense benchmarks. Higher is better.}
  \label{tab:3bit-llama-commonsense}
  \footnotesize
  \setlength{\tabcolsep}{6pt}
  \renewcommand{\arraystretch}{1.15}

\resizebox{0.8\linewidth}{!}{%
\begin{tabular}{cc*{6}{c}} %
\toprule
\textbf{Model} & \textbf{Method}
  & \textbf{WinoGrande}\,$\uparrow$
  & \textbf{PiQA}\,$\uparrow$
  & \textbf{HellaSwag}\,$\uparrow$
  & \textbf{ArcE}\,$\uparrow$
  & \textbf{ArcC}\,$\uparrow$
  & \textbf{Average}\,$\uparrow$ \\
\midrule

\multirow{4}{*}{LLaMA-7B}
  & FP16             & 69.93 & 78.67 & 56.97 & 75.21 & 41.89 & 64.53 \\
  & \cellcolor{gray!30}VQRound+RTN  & \cellcolor{gray!30}\textbf{65.75} & \cellcolor{gray!30}73.23 & \cellcolor{gray!30}49.69 & \cellcolor{gray!30}65.32 & \cellcolor{gray!30}31.23 & \cellcolor{gray!30}57.06 \\
  & GPTQ             & 65.43 & \textbf{74.27} & \textbf{51.62} & \textbf{67.34} & \textbf{34.90} & \textbf{58.72} \\
  & \cellcolor{gray!30}VQRound+{GPTQ} & \cellcolor{gray!30}64.25 & \cellcolor{gray!30}74.16 & \cellcolor{gray!30}50.62 & \cellcolor{gray!30}66.58 & \cellcolor{gray!30}32.34 & \cellcolor{gray!30}57.59 \\
\specialrule{.1em}{.25em}{.25em}

\multirow{4}{*}{LLaMA-13B}
  & FP16             & 72.77 & 79.16 & 59.92 & 77.40 & 46.42 & 67.13 \\
  & \cellcolor{gray!30}VQRound+RTN  & \cellcolor{gray!30}67.72 & \cellcolor{gray!30}\textbf{77.26} & \cellcolor{gray!30}53.62 & \cellcolor{gray!30}67.38 & \cellcolor{gray!30}\textbf{39.33} & \cellcolor{gray!30}61.90 \\
  & GPTQ             & \textbf{69.53} & 76.55 & 53.45 & 72.26 & 39.25 & \textbf{62.21} \\
  & \cellcolor{gray!30}VQRound+{GPTQ} & \cellcolor{gray!30}68.19 & \cellcolor{gray!30}76.88 & \cellcolor{gray!30}\textbf{54.01} & \cellcolor{gray!30}\textbf{72.43} & \cellcolor{gray!30}38.65 & \cellcolor{gray!30}62.03 \\
\specialrule{.1em}{.25em}{.25em}

\multirow{4}{*}{LLaMA2-7B}
  & FP16             & 69.06 & 78.07 & 57.13 & 76.30 & 43.43 & 64.80 \\
  & \cellcolor{gray!30}VQRound+RTN & \cellcolor{gray!30}63.77 & \cellcolor{gray!30}74.21 & \cellcolor{gray!30}49.59 & \cellcolor{gray!30}67.97 & \cellcolor{gray!30}33.11 & \cellcolor{gray!30}57.73 \\
  & GPTQ  & 62.51 & 72.91 & 49.31 & 65.82 & 33.79 & 56.87 \\
  & \cellcolor{gray!30}VQRound+{GPTQ}   & \cellcolor{gray!30}\textbf{65.59} & \cellcolor{gray!30}\textbf{74.59} & \cellcolor{gray!30}\textbf{49.71} & \cellcolor{gray!30}\textbf{70.88} & \cellcolor{gray!30}\textbf{38.31} & \cellcolor{gray!30}\textbf{59.82} \\
\specialrule{.1em}{.25em}{.25em}

\multirow{4}{*}{LLaMA2-13B}
  & FP16             & 72.38 & 79.05 & 60.07 & 79.38 & 48.46 & 67.87 \\
  & \cellcolor{gray!30}VQRound+RTN & \cellcolor{gray!30}66.93 & \cellcolor{gray!30}\textbf{76.28} & \cellcolor{gray!30}53.81 & \cellcolor{gray!30}71.72 & \cellcolor{gray!30}38.74 & \cellcolor{gray!30}61.50 \\
  & GPTQ             & 68.51 & 76.12 & 53.75 & 72.01 & 39.51 & 61.98 \\
  & \cellcolor{gray!30}VQRound+{GPTQ}   & \cellcolor{gray!30}\textbf{69.14} & \cellcolor{gray!30}76.17 & \cellcolor{gray!30}\textbf{55.03} & \cellcolor{gray!30}\textbf{73.06} & \cellcolor{gray!30}\textbf{39.68} & \cellcolor{gray!30}\textbf{62.62} \\
\bottomrule
\end{tabular}
}
\vspace{-2mm}
\end{table*}

\subsection{Ablation Study}

\paragraph{Residual Initialization}
We study the effect of Hessian guidance in residual initialization. 
Specifically, we compare two integer estimation strategies under the same quantization framework: initialization without curvature information and initialization incorporating Hessian statistics $\mathbf{H}$. Tab.~\ref{tab:init-residual} reports perplexity under both soft and hard rounding. 
Across model sizes and rounding directions, incorporating Hessian information consistently reduces perplexity compared to its non-Hessian counterpart. These results indicate that curvature-aware initialization provides a more suitable starting point for residual optimization, leading to improved quantization performance.

We further examine the effectiveness of VQ during the initialization. As shown in Tab.~\ref{tab:init_comparison} and Tab.~\ref{tab:qwen3_init_comparison}, VQ-based initialization consistently yields lower perplexity than alternative methods, revealing that VQ achieves superior optimization starting point.

\begin{table}[ht]
  \centering
  \caption{Perplexity comparison of residual integer initialization strategies with and without Hessian guidance under soft and hard rounding in OPT models.}
  \label{tab:init-residual}
  \small
  \renewcommand{\arraystretch}{1.2}
  \setlength{\tabcolsep}{10pt}
  
  \begin{tabular}{l cc cc}
    \toprule
    \multirow{3}{*}{\textbf{Init Method}}
      & \multicolumn{4}{c}{\textbf{OPT Model Size}} \\
    \cmidrule(lr){2-5}
      & \multicolumn{2}{c}{\textbf{125M}}
      & \multicolumn{2}{c}{\textbf{350M}} \\
    \cmidrule(lr){2-3}\cmidrule(lr){4-5}
      & \textbf{Soft} & \textbf{Hard}
      & \textbf{Soft} & \textbf{Hard} \\
    \midrule
    $W_{\text{Q}}/\text{s}$                    & 63.54 & 58.04 & 23.95 & 27.92 \\
    $W_{\text{Q}}/\text{s}$ w.\ $\mathbf{H}$    & 46.11 & 40.85 & 23.42 & 24.36 \\
    \bottomrule
  \end{tabular}
  \vspace{-2mm}
\end{table}

\textbf{Parameter Efficiency.} We compare the number of trainable parameters required by AdaRound and VQRound across different model scales. As shown in Tab.~\ref{tab:trainables-round}, VQRound introduces substantially fewer trainable parameters than AdaRound. For example, on LLaMA-13B, VQRound requires only 0.07\% of the parameters used by AdaRound. This reduction becomes more pronounced as model size increases, indicating that VQRound scales more favorably to larger models. Despite the significant parameter savings, VQRound maintains competitive performance, demonstrating that VQ reparameterization provides an efficient alternative to element-wise rounding optimization. 

\begin{table}[h]
  \centering
  \vspace{-2mm}
  \caption{Trainable parameters comparison between AdaRound and VQRound across model scales.}
  \label{tab:trainables-round}
  \small
  \renewcommand{\arraystretch}{1.12}
  \setlength{\tabcolsep}{6pt}
  
  \begin{tabular}{l cc c}
    \toprule
    \textbf{Model} & \multicolumn{2}{c}{\textbf{Trainable Params}} & \textbf{Ratio (\%)} \\
    \cmidrule(lr){2-3}
                    & \textbf{AdaRound} & \textbf{VQRound} & \textbf{VQ/Ada} \\
    \midrule
    OPT-1.3B  & 1.21B  & 4.72M  & 0.39\% \\
    OPT-2.7B  & 2.16B  & 6.29M  & 0.29\% \\
    LLaMA-7B  & 6.48B  & 7.34M  & 0.11\% \\
    LLaMA-13B & 12.69B & 9.18M  & 0.07\% \\
    \bottomrule
  \end{tabular}
  \vspace{-2mm}
\end{table}

\textbf{Codebook Design.}
We further study the impact of the codebook size $k$ and vector dimension $d$ on quantization performance. Tab.~\ref{tab:codebook_size} reports 4-bit results on OPT-1.3B. The configuration with $k=2^{12}$ and $d=8$ achieves the lowest perplexity on WikiText2, while remaining competitive on C4. Increasing $k$ slightly increases the number of trainable parameters without consistent performance gains, whereas reducing $d$ leads to less expressive residual representation. Overall, $k=2^{12}$ and $d=8$ provide a favorable balance between model capacity and efficiency, and is adopted as the default setting in subsequent experiments.

\begin{table}[h]
  \centering
  \caption{VQRound 4-bit results under different codebook settings on OPT-1.3B. $k$ denotes the number of centroids and $d$ the vector dimension.}
  \label{tab:codebook_size}
  \small
  \renewcommand{\arraystretch}{1.12}
  \setlength{\tabcolsep}{13pt}
  
  \begin{tabular}{l cc}
    \toprule
    \multirow{2}{*}{\textbf{Codebook Setting}} & \multicolumn{2}{c}{\textbf{PPL}} \\
    \cmidrule(lr){2-3}
     & \textbf{WikiText2} & \textbf{C4} \\
    \midrule
    $k=2^{12}$, $d=4$ & 15.84 & \textbf{17.10} \\
    $k=2^{12}$, $d=8$ & \textbf{15.48} & 17.28 \\
    $k=2^{16}$, $d=4$ & 15.73 & 17.11 \\
    $k=2^{16}$, $d=8$ & 16.13 & 17.23 \\
    \bottomrule
  \end{tabular}
  \vspace{-2mm}
\end{table}

\section{Conclusion}
\label{conclusion}
We revisit adaptive rounding from an efficiency perspective and propose VQRound, a parameter-efficient reparameterization method that represents the rounding matrix using a compact vector-quantized codebook. Unlike low-rank optimization methods, VQRound directly controls the element-wise worst-case error under the $L_\infty$ norm, which is particularly important for heavy-tailed weight distributions in large language models. In addition to blockwise reconstruction, VQRound can be optimized through an E2E fine-tuning that enables cross-layer error compensation using a small calibration set. 
Extensive experiments across multiple model families demonstrate that VQRound achieves better convergence than adaptive rounding under the same optimization budget, while requiring only a small fraction of trainable parameters. These results suggest that adaptive rounding can be made both scalable and efficient through structured reparameterization.

\section*{Impact Statement}
This paper is working on advancing the field of Machine Learning. We have discussed the potential impact of our work in \S\ref{conclusion}. While our work may have various societal implications, we do not identify any specific consequences that require individual highlighting beyond the impacts of method efficiency.

\bibliography{main}
\bibliographystyle{icml2026}
\appendix
\onecolumn
\newpage
\appendix
\section{Appendix}
\subsection{Additional Experiment Result}
\label{appendix:results}

Tab.~\ref{tab:opt-c4} reports perplexity on C4 across the OPT model family under 4-bit and 3-bit quantization. 
Across model scales, VQRound consistently improves over vanilla RTN and further reduces perplexity when combined with GPTQ or OmniQuant in most settings. These results demonstrate that VQRound effectively reduces quantization error and remains robust under aggressive bit-width reduction.

\begin{table}[H]
    \centering
    \caption{C4 Perplexity in OPT model family. Lower is better.}
    \label{tab:opt-c4}
    \footnotesize
    \renewcommand{\arraystretch}{1.12}
    \setlength{\tabcolsep}{10.0pt}
    \begin{tabular}{lc
                    cccccc}
    \toprule
    \multirow{2}{*}{\textbf{Precision}} & \multirow{2}{*}{\textbf{Method}}
    & \multicolumn{6}{c}{\textbf{OPT Model Size}} \\
    \cmidrule(lr){3-8}
    & & \textbf{125M} & \textbf{350M} & \textbf{1.3B} & \textbf{2.7B} & \textbf{6.7B} & \textbf{13B} \\
    \midrule
    \multirow{1}{*}{FP16}
    & Baseline
    & 26.56 & 22.59 & 16.07 & 14.34 & 12.71 & 12.06  \\
    \specialrule{1pt}{3pt}{3pt}
    
    \multirow{6}{*}{4 bits}
    & RTN                   & 33.91 & 26.21 & 24.51 & 18.43 & 14.36 & 13.36 \\
    & \cellcolor{gray!30}VQRound+RTN & \cellcolor{gray!30}28.79 & \cellcolor{gray!30}24.39 & \cellcolor{gray!30}17.28 & \cellcolor{gray!30}15.27 & \cellcolor{gray!30}13.27 & \cellcolor{gray!30}12.53 \\
    & GPTQ                  & 29.22 & 24.63 & 16.97 & 15.00 & 13.18 & \textbf{12.26}\\
    & \cellcolor{gray!30}VQRound+{GPTQ}  & \cellcolor{gray!30}\textbf{28.72} & \cellcolor{gray!30}\textbf{23.44} & \cellcolor{gray!30}\textbf{16.80} & \cellcolor{gray!30}\textbf{14.87} & \cellcolor{gray!30}\textbf{13.01} & \cellcolor{gray!30}12.29\\
    \cmidrule(lr){2-8}
    & OmniQuant & 26.90 & - & 15.82 & 14.08 & 12.68 & 12.18 \\
    & \cellcolor{gray!30}VQRound+{OmniQuant}   & \cellcolor{gray!30}\best{26.10} & \cellcolor{gray!30}\best{-} & \cellcolor{gray!30}\best{15.46} & \cellcolor{gray!30}\best{13.73} & \cellcolor{gray!30}\best{12.15} & \cellcolor{gray!30}\best{11.65} \\
    \specialrule{1pt}{3pt}{3pt} 
    \multirow{6}{*}{3 bits}
    & RTN                   & 834 & 55.49 & 5.2e3 & 1.1e4 & 5.3e3 & 3.1e3 \\
    & \cellcolor{gray!30}VQRound+RTN & \cellcolor{gray!30}39.76 & \cellcolor{gray!30}31.40 & \cellcolor{gray!30}22.57 & \cellcolor{gray!30}19.28 & \cellcolor{gray!30}15.57 & \cellcolor{gray!30}14.37 \\
    & GPTQ                  & 42.41 & 31.33 & 21.63 & 18.17 & 17.14 & 13.34 \\
    & \cellcolor{gray!30}VQRound+{GPTQ}  & \cellcolor{gray!30}\textbf{38.87} & \cellcolor{gray!30}\textbf{27.13} & \cellcolor{gray!30}\textbf{20.02} & \cellcolor{gray!30}\textbf{17.16} & \cellcolor{gray!30}\textbf{14.25} & \cellcolor{gray!30}\textbf{13.18} \\
    \cmidrule(lr){2-8}
    & OmniQuant & 36.43 & - & 19.57 & 18.16 & 14.87 & 16.37 \\
    & \cellcolor{gray!30}VQRound+{OmniQuant}   & \cellcolor{gray!30}\best{31.34} & \cellcolor{gray!30}\best{-} & \cellcolor{gray!30}\best{18.58} & \cellcolor{gray!30}\best{15.31} & \cellcolor{gray!30}\best{13.25} & \cellcolor{gray!30}\best{12.96} \\
    \bottomrule
    \end{tabular}
\end{table}

Tab.~\ref{tab:4bit-llama-commonsense} presents 4-bit zero-shot accuracy on common-sense reasoning benchmarks. 
VQRound maintains strong downstream performance across all evaluated LLaMA and LLaMA2 models. 
Compared to RTN, VQRound consistently improves average accuracy, while remaining competitive with GPTQ. 
Notably, VQRound or VQRound+GPTQ achieves the highest or near-highest average accuracy in multiple tasks, especially on LLaMa2-13B, indicating that the proposed reparameterization enhances quantization quality without sacrificing generalization. 
Overall, these results confirm that VQRound preserves both language modeling fidelity and downstream task performance.

\begin{table}[!htb]
    \centering
    \setlength{\tabcolsep}{20pt}
    \caption{4 bit zero-shot accuracy (\%) on commonsense benchmarks. Higher is better.}
    \label{tab:4bit-llama-commonsense}
    \footnotesize
    \setlength{\tabcolsep}{6pt}
    \renewcommand{\arraystretch}{1.15}
    
    \resizebox{0.95\linewidth}{!}{%
    \begin{tabular}{cc*{6}{c}} %
    \toprule
    \textbf{Model} & \textbf{Method}
      & \textbf{WinoGrande}\,$\uparrow$
      & \textbf{PiQA}\,$\uparrow$
      & \textbf{HellaSwag}\,$\uparrow$
      & \textbf{ArcE}\,$\uparrow$
      & \textbf{ArcC}\,$\uparrow$
      & \textbf{Average}\,$\uparrow$ \\
    \midrule
    
    \multirow{4}{*}{LLaMA-7B}
      & FP16             & 69.93 & 78.67 & 56.97 & 75.21 & 41.89 & 64.53 \\
      & \cellcolor{gray!30}VQRound+RTN  & \cellcolor{gray!30}\textbf{70.01} & \cellcolor{gray!30}77.53 & \cellcolor{gray!30}55.36 & \cellcolor{gray!30}73.86 & \cellcolor{gray!30}39.25 & \cellcolor{gray!30}63.20 \\
      & GPTQ             & 69.93 & 77.86 & \textbf{55.99} & \textbf{74.12} & 39.51 & \textbf{63.48} \\
      & \cellcolor{gray!30}VQRound+{GPTQ} & \cellcolor{gray!30}68.59 & \cellcolor{gray!30}\textbf{78.18} & \cellcolor{gray!30}55.17 & \cellcolor{gray!30}73.44 & \cellcolor{gray!30}\textbf{40.10} & \cellcolor{gray!30}63.10 \\
    \specialrule{.1em}{.25em}{.25em}
    
    \multirow{4}{*}{LLaMA-13B}
      & FP16             & 72.77 & 79.16 & 59.92 & 77.40 & 46.42 & 67.13 \\
      & \cellcolor{gray!30}VQRound+RTN  & \cellcolor{gray!30}71.59 & \cellcolor{gray!30}78.51 & \cellcolor{gray!30}58.45 & \cellcolor{gray!30}\textbf{76.52} & \cellcolor{gray!30}45.05 & \cellcolor{gray!30}66.02 \\
      & GPTQ          & 72.77 & \textbf{79.11} & \textbf{58.98} & 76.26 & 45.39 & \textbf{66.50} \\
      & \cellcolor{gray!30}VQRound+{GPTQ} & \cellcolor{gray!30}\textbf{72.85} & \cellcolor{gray!30}78.84 & \cellcolor{gray!30}58.78 & \cellcolor{gray!30}75.72 & \cellcolor{gray!30}\textbf{45.65} & \cellcolor{gray!30}66.37 \\
    \specialrule{.1em}{.25em}{.25em}
    
    \multirow{4}{*}{LLaMA2-7B}
      & FP16             & 69.06 & 78.07 & 57.13 & 76.30 & 43.43 & 64.80 \\
      & \cellcolor{gray!30}VQRound+RTN & \cellcolor{gray!30}68.11 & \cellcolor{gray!30}76.88 & \cellcolor{gray!30}55.55 & \cellcolor{gray!30}73.36 & \cellcolor{gray!30}40.27 & \cellcolor{gray!30}62.83 \\
      & GPTQ  & \textbf{68.59} & 76.88 & \textbf{55.87} & \textbf{75.13} & \textbf{41.13} & \textbf{63.52} \\
      & \cellcolor{gray!30}VQRound+{GPTQ}   & \cellcolor{gray!30}68.35 & \cellcolor{gray!30}\textbf{77.20} & \cellcolor{gray!30}55.47 & \cellcolor{gray!30}73.86 & \cellcolor{gray!30}40.27 & \cellcolor{gray!30}63.03 \\
    \specialrule{.1em}{.25em}{.25em}
    
    \multirow{4}{*}{LLaMA2-13B}
      & FP16             & 72.38 & 79.05 & 60.07 & 79.38 & 48.46 & 67.87 \\
      & \cellcolor{gray!30}VQRound+RTN & \cellcolor{gray!30}\textbf{72.22} & \cellcolor{gray!30}\textbf{78.94} & \cellcolor{gray!30}\textbf{59.21} & \cellcolor{gray!30}77.65 & \cellcolor{gray!30}\textbf{45.90} & \cellcolor{gray!30}\textbf{66.78} \\
      & GPTQ             & 70.96 & 78.02 & 58.74 & 77.44 & \textbf{45.90} & 66.21 \\
      & \cellcolor{gray!30}VQRound+{GPTQ}   & \cellcolor{gray!30}72.14 & \cellcolor{gray!30}78.73 & \cellcolor{gray!30}59.14 & \cellcolor{gray!30}\textbf{78.11} & \cellcolor{gray!30}45.39 & \cellcolor{gray!30}66.70 \\
    \bottomrule
\end{tabular}
}
\vspace{-2mm}
\end{table}

Tab.~\ref{tab:init_comparison} and Tab.~\ref{tab:qwen3_init_comparison} compare different reparameterization strategies after rounding initialization across the LLaMA family and Qwen3 family. VQ consistently achieves the lowest perplexity under both soft and hard rounding settings, while LoRA and Kronecker variants often exhibit noticeably higher perplexity or even instability which leads to value out of range (NaN). These results indicate that VQ produces substantially smaller initialization error than low-rank or Kronecker-based reparameterizations. In particular, the gap becomes more pronounced under hard rounding, suggesting that VQ better controls element-wise deviations and is more robust to local residual amplification. Overall, the results demonstrate that the VQ-based reparameterization provides a more faithful and stable initialization for subsequent optimization.

\begin{table}[!htb]
\centering
\setlength{\tabcolsep}{20pt}
\caption{Comparison of perplexity score of different reparameterization methods (VQ, LoRA, Kronecker) after rounding initialization across LLaMA and LLaMA2 models. Lower is better.}
\label{tab:init_comparison}
\resizebox{\textwidth}{!}{%
\begin{tabular}{llcccc}
\toprule
\multirow{2}{*}{Model} & \multirow{2}{*}{Method} & \multicolumn{2}{c}{Soft} & \multicolumn{2}{c}{Hard} \\ \cline{3-6} 
 &  & Wikitext2 & C4 & Wikitext2 & C4 \\ \midrule
\multirow{5}{*}{LLaMA 7B} & VQ & \textbf{5.78} & \textbf{7.46} & \textbf{6.38} & \textbf{8.19} \\
 & LoRA(Kaiming) & 11.08 & 16.01 & NaN & NaN \\
 & LoRA(SVD) & 6.70 & 8.89 & 10.08 & 13.54 \\
 & Kronecker(Kaiming) & 11.08 & 16.01 & NaN & NaN \\
 & Kronecker(SVD) & 11.09 & 16.06 & 292.67 & 270.67 \\ \midrule
\multirow{5}{*}{LLaMA 13B} & VQ & \textbf{5.17} & \textbf{6.88} & \textbf{5.53} & \textbf{7.29} \\
 & LoRA(Kaiming) & 7.24 & 10.56 & NaN & NaN \\
 & LoRA(SVD) & 5.77 & 7.79 & 7.38 & 10.02 \\
 & Kronecker(Kaiming) & 7.24 & 10.56 & NaN & NaN \\
 & Kronecker(SVD) & 7.24 & 10.58 & 123.46 & 120.51 \\ \midrule
\multirow{5}{*}{LLaMA2 7B} & VQ & \textbf{5.57} & \textbf{7.38} & \textbf{6.17} & \textbf{8.22} \\
 & LoRA(Kaiming) & NaN & 59.97 & NaN & NaN \\
 & LoRA(SVD) & 7.03 & 9.92 & 17.80 & 24.64 \\
 & Kronecker(Kaiming) & NaN & 59.97 & NaN & NaN \\
 & Kronecker(SVD) & NaN & 59.40 & 1338.83 & 928.37 \\ \midrule
\multirow{5}{*}{LLaMA2 13B} & VQ & \textbf{4.96} & \textbf{6.82} & \textbf{5.25} & \textbf{7.23} \\
 & LoRA(Kaiming) & 15.78 & 25.49 & NaN & NaN \\
 & LoRA(SVD) & 5.78 & 8.09 & 8.14 & 12.79 \\
 & Kronecker(Kaiming) & 15.78 & 25.49 & NaN & NaN \\
 & Kronecker(SVD) & 15.83 & 25.45 & 446.28 & 394.82 \\ \bottomrule
\end{tabular}%
}
\end{table}

\begin{table}[!htb]
\centering
\setlength{\tabcolsep}{20pt}
\caption{Comparison of perplexity score of different reparameterization methods (VQ, LoRA, Kronecker) after rounding initialization across the Qwen3 model family. Lower is better.}
\label{tab:qwen3_init_comparison}
\resizebox{\textwidth}{!}{%
\begin{tabular}{llcccc}
\toprule
\multirow{2}{*}{Model} & \multirow{2}{*}{Method} & \multicolumn{2}{c}{Soft} & \multicolumn{2}{c}{Hard} \\ \cline{3-6} 
 &  & Wikitext2 & C4 & Wikitext2 & C4 \\ \midrule
\multirow{5}{*}{Qwen3 0.6B} & VQ & \textbf{22.84} & \textbf{32.60} & \textbf{41.19} & \textbf{56.33} \\
 & LoRA(Kaiming) & NaN & NaN & NaN & NaN \\
 & LoRA(SVD) & 113.65 & 138.84 & 1436.11 & 1759.19 \\
 & Kronecker(Kaiming) & NaN & NaN & NaN & NaN \\
 & Kronecker(SVD) & NaN & NaN & NaN & NaN \\ \midrule
\multirow{5}{*}{Qwen3 1.7B} & VQ & \textbf{19.28} & \textbf{24.77} & \textbf{28.64} & \textbf{32.27} \\
 & LoRA(Kaiming) & 3448.73 & 2245.94 & NaN & NaN \\
 & LoRA(SVD) & 58.72 & 66.72 & 3181.12 & 1798.36 \\
 & Kronecker(Kaiming) & 3448.73 & 2245.94 & NAN & NaN \\
 & Kronecker(SVD) & 2618.89 & 1669.90 & 1323702.75 & 979412.50 \\ \midrule
\multirow{5}{*}{Qwen3 4B} & VQ & \textbf{13.97} & \textbf{20.62} & \textbf{17.33} & \textbf{25.11} \\
 & LoRA(Kaiming) & 26.10 & 41.61 & NaN & NaN \\
 & LoRA(SVD) & 20.41 & 31.70 & 144.53 & 199.08 \\
 & Kronecker(Kaiming) & 26.10 & 41.61 & NaN & NaN \\
 & Kronecker(SVD) & 26.10 & 41.61 & 3747.27 & 4061.17 \\ \midrule
\multirow{5}{*}{Qwen3 8B} & VQ & \textbf{10.03} & \textbf{16.04} & \textbf{11.74} & \textbf{18.44} \\
 & LoRA(Kaiming) & 23.59 & 34.13 & NaN & NaN \\
 & LoRA(SVD) & 15.87 & 24.18 & 100.09 & 140.51 \\
 & Kronecker(Kaiming) & 23.59 & 34.13 & NaN & NaN \\
 & Kronecker(SVD) & 23.62 & 34.12 & 385.47 & 405.06 \\ \bottomrule
\end{tabular}%
}
\end{table}

Tab.~\ref{tab:e2e_compare} compares VQRound+RTN and AdaRound+RTN under end-to-end finetuning on WikiText2 and C4. Across all feasible model sizes, VQRound consistently achieves lower perplexity than AdaRound, with the performance gap widening as model scale increases. Notably, while VQRound remains stable up to OPT-13B, AdaRound runs out of memory starting from OPT-2.7B on a single NVIDIA A6000 GPU. This highlights not only the improved quantization quality of VQRound, but also its superior scalability under practical hardware constraints. The results confirm that the proposed reparameterization significantly reduces the optimization footprint while maintaining or improving accuracy.

\begin{table}[htbp]
\centering
\caption{Comparison of VQRound+RTN and AdaRound+RTN under E2E finetuning. 
Perplexity is reported on WikiText2 and C4. AdaRound runs out of memory (OOM) starting from OPT-2.7B on a single NVIDIA A6000.}
\label{tab:e2e_compare}
\begin{tabular}{lcccc}
\toprule
\multirow{2}{*}{Model} & \multicolumn{2}{c}{VQRound+RTN} & \multicolumn{2}{c}{AdaRound+RTN} \\
\cmidrule(lr){2-3} \cmidrule(lr){4-5}
 & WikiText2 & C4 & WikiText2 & C4 \\
\midrule
OPT-125M & 30.69 & \textbf{28.79} & \textbf{30.61} & 28.92 \\
OPT-350M & \textbf{23.77} & \textbf{24.39} & 26.08 & 25.72 \\
OPT-1.3B  & \textbf{15.48} & \textbf{17.28} & 18.33 & 18.74 \\
OPT-2.7B  & \textbf{13.30} & \textbf{15.27} & \textbf{OOM} & \textbf{OOM} \\
\bottomrule
\end{tabular}
\end{table}

Tab.~\ref{tab:reconstruction_strategy_full} compares blockwise reconstruction and end-to-end (E2E) finetuning across OPT, LLaMA, and LLaMA2 model families. The two approaches achieve comparable performance overall, with E2E finetuning slightly outperforming block-wise reconstruction on most model sizes, particularly for medium and large OPT models. These results indicate that VQRound supports both block-wise and global optimization strategies with competitive effectiveness.

\begin{table}[htbp]
\centering
\caption{Perplexity (PPL) of VQRound End-to-End finetuning vs. Block-wise reconstruction across OPT, LLaMA and LLaMA2 model families. Lower is better.}
\label{tab:reconstruction_strategy_full}
\renewcommand{\arraystretch}{1.1} 
\setlength{\tabcolsep}{6pt}    
\resizebox{0.7\linewidth}{!}{%
\begin{tabular}{ll cc cc}
\toprule
\multirow{2}{*}{\textbf{Model Family}} & \multirow{2}{*}{\textbf{Size}} & \multicolumn{2}{c}{\textbf{Block-wise}} & \multicolumn{2}{c}{\textbf{End-to-End}} \\
\cmidrule(lr){3-4} \cmidrule(lr){5-6}
& & ~~~~~WikiText2~~~~~ & ~~~~~C4~~~~~ & ~~~~~WikiText2~~~~~ & ~~~~~C4~~~~~ \\
\midrule
\multirow{6}{*}{OPT} 
& 125M & 31.08 & 28.97 & \textbf{30.69} & \textbf{28.79} \\
& 350M & \textbf{23.44} & \textbf{24.05} & 23.77 & 24.39 \\
& 1.3B & 15.84 & 17.31 & \textbf{15.48} & \textbf{17.28} \\
& 2.7B & \textbf{13.24} & \textbf{15.27} & 13.30 & \textbf{15.27} \\
& 6.7B & 11.30 & 13.35 & \textbf{11.26} & \textbf{13.27} \\
& 13B  & 10.78 & 12.56 & \textbf{10.66} & \textbf{12.53} \\
\midrule
\multirow{2}{*}{LLaMA} 
& 7B   & 6.18 & 7.89 & \textbf{6.13} & \textbf{7.88} \\
& 13B  & \textbf{5.40} & 7.18 & 5.42 & \textbf{7.17} \\
\midrule
\multirow{2}{*}{LLaMA2} 
& 7B   & 6.00 & 7.91 & \textbf{5.90} & \textbf{7.88} \\
& 13B  & \textbf{5.19} & 7.14 & \textbf{5.19} & \textbf{7.13} \\
\bottomrule
\end{tabular}
}
\end{table}

\clearpage
\subsection{Pseudocode}
\label{algo}

We summarize the proposed Hessian-aware residual initialization in Alg.~\ref{alg:hessian_init} and the overall training procedure of VQRound in Alg.~\ref{alg:finetune}.

\begin{algorithm}[h]
\caption{Hessian-aware Rounding Matrix Initialization}
\label{alg:hessian_init}
\begin{algorithmic}[1]

\REQUIRE Weight matrix $W$; Hessian matrix $\mathbf{H}$; blocksize; percdamp; per-channel quantization parameters quantizer, $(scale, zero, q_{max})$
\ENSURE Quantized weights $W_Q$; Base integer matrix $B$; Rounding matrix $\tilde{H}$

\STATE // Precompute damped Hessian inverse factor
\STATE $damp \leftarrow percdamp \cdot \mathrm{mean}(\mathrm{diag}(\mathbf{H}))$
\STATE Add $damp$ to diagonal of $\mathbf{H}$
\STATE $\mathbf{H} \leftarrow \mathrm{cholesky}(\mathbf{H})$
\STATE $\mathbf{H} \leftarrow \mathrm{cholesky\_inverse}(\mathbf{H})$
\STATE $\mathbf{H} \leftarrow \mathrm{cholesky}(\mathbf{H}, upper=True)$
\STATE $\mathbf{H}_{inv} \leftarrow \mathbf{H}$

\STATE Initialize $W_Q \leftarrow 0$, $B \leftarrow 0$, $R \leftarrow 0$

\FOR{$i_1 = 0$ to $columns$ step blocksize}
    \STATE $i_2 \leftarrow \min(i_1 + blocksize, columns)$
    \STATE $W_1 \leftarrow W[:, i_1:i_2]$
    \STATE $\mathbf{H}_{inv_1}  \leftarrow \mathbf{H}_{inv}[i_1:i_2, i_1:i_2]$
    \STATE Initialize $Error$

    \FOR{$j = 0$ to $(i_2-i_1-1)$}
        \STATE $w \leftarrow W_1[:, j]$
        \STATE $d \leftarrow \mathbf{H}_{inv_1} [j,j]$

        \STATE $q \leftarrow quantizer(w, scale, zero, q_{max})$
        \STATE $W_Q[:, i_1+j] \leftarrow q$

        \STATE $err \leftarrow (w - q)/d$

        \STATE $W_1[:, j:] \leftarrow W_1[:, j:] - err \cdot \mathbf{H}_{inv_1} [j, j:]$

        \STATE Store $err$ in $Error$

        \STATE $u \leftarrow w/scale$
        \STATE $b\_col \leftarrow \lfloor u \rfloor$
        \STATE $B[:, i_1+j] \leftarrow b\_col$
        \STATE $\tilde{H}[:, i_1+j] \leftarrow \mathrm{clip}(u - b\_col - err,\ 0,1)$

    \ENDFOR

    \IF{$i_2 < columns$}
        \STATE $W[:, i_2:] \leftarrow W[:, i_2:] - Error \cdot \mathbf{H}_{inv}[i_1:i_2, i_2:]$
    \ENDIF

\ENDFOR

\STATE \textbf{return} $W_Q$, $B$, $\tilde{H}$

\end{algorithmic}
\end{algorithm}

\paragraph{Hessian-aware residual initialization.}
Before end-to-end optimization, we construct an informed initialization of the rounding matrix by incorporating second-order curvature information. As detailed in Alg.~\ref{alg:hessian_init}, we sequentially quantize weight columns using RTN and compute their residuals. Each residual is scaled by the corresponding diagonal entry of the damped Hessian inverse factor and propagated to the remaining unprocessed columns, following a GPTQ-style second-order error compensation scheme. 

The corrected residual is then converted into a base integer component $B$ and a fractional residual $R \in [0,1]$. The residual $R$ is then used for VQ reparameterization. Since the zero-point is integer-valued, it does not affect the fractional component and is therefore omitted in constructing $B$ and $R$. This Hessian-aware initialization substantially reduces the initial deviation from the full-precision weights, stabilizes subsequent optimization, and improves final quantization performance.

\paragraph{End-to-end codebook optimization.}
Starting from the initialized rounding signal, we optimize the VQ codebook parameters in an end-to-end fashion (Alg.~\ref{alg:finetune}). All model parameters are frozen except the codebook, and we minimize a distillation objective between teacher and student logits, combined with a rounding regularizer. 

\begin{algorithm}[H]
    \caption{VQRound end-to-end finetuning}
    \label{alg:finetune}
    \small
    \begin{algorithmic}
        \REQUIRE Teacher model $\mathcal{M}_t$, Student model $\mathcal{M}_s$; Frozen FP weights $\mathbf{W}$, per-channel scale $s$, zero-point $Z$; fixed VQ indices $I=\{i_\ell\}_{\ell=1}^{L}$; initial codebook $\mathcal{C}=\{c_k\}_{k=1}^{K}$; calibration dataset $\mathcal{D}$; temperature $T$; rectified sigmoid $h(\cdot)$; rounding regularizer $\mathcal{R}(\cdot)$; rounding regularization weight $\lambda$; steps $N$; Anneal parameter $(\beta_{\text{high}}, \beta_{\text{low}})$
        \STATE \textbf{Freeze} all params of $\mathcal{M}_s$ except the codebook 
        \STATE \textbf{Init Adam optimizer} on $\mathcal{C}$
        \FOR{$t \gets 1$ \textbf{to} $N$}
            \STATE $x \gets \mathrm{NextSample}(\mathcal{D})$; \text{Batch size = 1}
            \STATE $\hat{A} \gets \mathrm{Reshape}([\,c_{i_1},\dots,c_{i_l}\,])$
            \STATE $\hat{W} \gets s\big(\mathrm{clip}(\lfloor W/s \rfloor + h(\hat{A}) + Z, \,0, \,Q_{\max})-Z\big)$
            \STATE $\hat{y} \gets \mathcal{M}_s(x;\hat{W})$;\quad $y \gets \mathcal{M}_t(x)$
            \STATE $\mathcal{L}_{\mathrm{KD}} \gets \mathrm{KL} \ \!\big(\mathrm{softmax}(\hat{y}/T)\,\|\,\mathrm{softmax}(y/T)\big)$
            \STATE $\beta_t \gets \mathrm{Anneal} \ \!\big(\beta_{\text{high}},\beta_{\text{low}},t,N\big)$
            \STATE $\mathcal{L} \gets \mathcal{L}_{\mathrm{KD}} + \lambda\,\mathcal{R}(\hat{H};\beta_t)$
            \STATE \textbf{Update} $\mathcal{C} \gets \mathrm{Adam}(\mathcal{C}, \nabla_{\mathcal{C}})$
        \ENDFOR
    \end{algorithmic}
\end{algorithm}

\subsection{Proof of Theorem and Corollary}
\label{sec:proof}

\subsubsection{Proof of Theorem~\ref{thm:lipschitz_contraction}}
\label{sec:proof:lipschitz_contraction}

\begin{proof}
Define $g(x)=\gamma+(\zeta-\gamma)\sigma(x)$. Its derivative satisfies
\[
g'(x)=(\zeta-\gamma)\sigma(x)(1-\sigma(x))\le (\zeta-\gamma)/4,
\]
since $\max_{u\in(0,1)}u(1-u)=1/4$. Hence $g$ is $\mathcal{C}$-Lipschitz with
$\mathcal{C}=(\zeta-\gamma)/4$. As $\mathrm{clip}(\cdot,0,1)$ is $1$-Lipschitz, their
composition $h$ is also $\mathcal{C}$-Lipschitz. Applying this bound entrywise yields the
result.
\end{proof}

\subsubsection{Proof of Corollary~\ref{cor:lipschitz_tail_transfer}}
\label{sec:proof:lipschitz_tail_transfer}

\begin{proof}
The event $\{|\Delta H_{ij}|>\varepsilon\}$ implies
$\{|\Delta{A}_{ij}|>\varepsilon/\mathcal{C}\}$ by
Eq.~\ref{eq:elemwise_lipschitz_main}. Taking probabilities yields the claim.
\end{proof}

\subsubsection{Proof of Theorem~\ref{thm:clipping_threshold}}
\label{sec:proof:lipschitz_contraction}
\begin{proof}
By Theorem~\ref{thm:lipschitz_contraction},
\[
|g(\tilde{A}_{ij})-g({A}{ij})|\le \mathcal{C}|\tilde{A}_{ij}-A{ij}|.
\]
If $|\tilde{A}_{ij}-A_{ij}|>\delta_{ij}/\mathcal{C}$ holds, this displacement exceeds the margin
$\delta_{ij}$, implying $g(\tilde{A}_{ij})\notin(0,1)$. Applying the clipping
operator yields $\tilde H_{ij}\in\{0,1\}$.
\end{proof}

\subsubsection{Proof of Corollary~\ref{cor:clipping_probability}}
\label{sec:proof:clipping_probability}

\begin{proof}
Theorem~\ref{thm:clipping_threshold} implies the event inclusion
$\{|\Delta{A}_{ij}|>\delta_{ij}/\mathcal{C}\}\subseteq\{\tilde H_{ij}\in\{0,1\}\}$.
Taking probabilities yields the result.
\end{proof}

\subsection{Detailed Analysis of SVD and Kronecker Product Decomposition}
\label{sec:svd_kronecker}

One intuitive approach to reducing the parameter count of the rounding matrix $H \in [0, 1]^{m \times n}$ is enforcing low-rankness as employed in CBQ~\cite{ding2025cbqcrossblockquantizationlarge}. These methods typically optimize a global energy metric, such as the Frobenius norm $\|E\|_{F}=\|H - \hat{H}\|_F$, where $\hat{H}$ is the low-rank approximation of $H$.

\textbf{Low-Rank Decomposition}: This is the most straightforward way to approximate the rounding matrix with a reduced set of parameters, \textit{e.g.}, $H_{\text{LR}} = AB^\top$, where $A \in \mathbb{R}^{m \times r}$ and $B \in \mathbb{R}^{n \times r}$. The approximation can be done via standard singular-value decomposition (SVD). Using the Eckart-Young-Mirsky theorem~\cite{eckart1936approximation}, we have the following approximation error:
\begin{equation}
\label{eq:lora_frobenius}
    \|E_{\text{LR}}\|_F = \|H - H_\text{LR}\| = (\sum_{k>r} \sigma_k^2)^{1/2} ,
\end{equation}
\begin{equation}
\label{eq:lora_l2}
    \|E_{\text{LR}}\|_2 = \|H - H_\text{LR}\|_2 = \sigma_{r+1},
\end{equation}
where $\|E_{\text{LR}}\|_2$ denotes the spectral norm.

\textbf{Kronecker Product Decomposition}: Kronecker product is another feasible way to approximate the rounding matrix $H$ with low parameter counts, \textit{i.e.}, $H_{kron} = A \otimes B$, where $A\in \mathbb{R}^{a\times c}$ and $B\in \mathbb{R}^{b\times d}$, $m=a\cdot b$ and $n=c\cdot d$.

Kronecker product is a relatively complex matrix operation, thus we denote its result through Van Loan-Pitsiantis rearrangement~\cite{vanloan1993approximation}: $R(A \otimes B) = vec(A)vec(B)^\top$. By
doing this, we can still optimize the Kronecker product according to Singular Value Decomposition (SVD): $R(H)=\sum^r_{k=1}\sigma_ku_kv^\top_k$. Due to the properties of this rearrangement~\cite{eckart1936approximation}:
\begin{equation}
    \|H\|_F = \|R(H)\|_F
\end{equation}
Based on the conclusion from Eq.~\ref{eq:lora_frobenius}, we found that its Frobenius norm satisfies:
\begin{equation}
    \|E_{Kron}\|_F=\|H-H_{Kron}\|_F=(\sum_{k=1}^r\sigma_k(R(H))^2)^{1/2}
\end{equation}
which is equivalent to a low-rank decomposition of rank $k$. Similarly, in the Kronecker decomposition, we are still minimizing its Frobenius norm with regard to the singular value.

To approximate the latent rounding matrix, energy-based metrics like the Frobenius norm can be misleading.

We emphasize the element-wise worst case error, defined as infinity norm $\|E\|_{\infty} \triangleq \max_{i,j}|E_{ij}|$. For any error matrix $E$,

\begin{equation}
\label{eq:norm}
\|E\|_{\infty} \;\le\; \|E\|_{2} \;\le\; \|E\|_{F},
\end{equation}
where $\|E\|_{2}$ denotes the spectral norm of the error matrix. $\|E\|_{\infty}$ is more strict than $\|E\|_2$ and $\|E\|_F$ in controlling the upper bound of loss. In practice, prioritizing a small $\|E\|_{\infty}$ (\textit{i.e.} minimizing the worst-case per-weight error) curbs outlier rounding decisions and has been found to correlate better with downstream performance.

\subsection{Experiment setting}
\label{appendix:exp_setting}
To facilitate reproducibility, we detail the experimental settings and hyperparameters. For VQ initialization, a codebook of 4096 ($2^{12}$) centroids with vector dimension $D=8$ is employed. Each layer undergoes 100 iterations of K-Means clustering, providing a balance between search quality and initialization efficiency. During end-to-end fine-tuning, the codebook is optimized with the Adam optimizer~\citep{adam}. A unified hyperparameter configuration is adopted across models: the learning rate is set to $1e-2$, the rounding regularization coefficient $\lambda$ to $1e-2$, and the annealing schedule for $\beta$ decreases from 20 to 2. Fine-tuning is conducted for 5000 steps, with the first 10\% used as a distillation-only warm-up phase to ensure stable convergence, after which the rounding loss is incorporated into the training objective. All experiments are performed on 128 randomly sampled sequences from the C4 dataset~\citep{c4dataset} with length 2048. To accelerate initialization, GPU-accelerated K-Means clustering is implemented using FAISS~\citep{douze2025faisslibrary}. All experiments are conducted on a single NVIDIA RTX A6000 GPU.

\subsection{Ackonwledgement}
The computation results in the paper are conducted in the Swiss National Supercomputing Center (CSCS). We thank for the relevant personnel and institutions for providing access and computational resources for our research.

\newpage
\subsection{Figure}
\label{sec:figure}

Fig.~\ref{fig:distribution_comparison} illustrates the error density distribution of the parameter matrix $A$ and the final rounding matrix $H$ under different reparameterization methods. In the error distribution of both metrics, VQRound (blue) exhibits significantly higher density near zero and a much narrower distribution range compared to LoRA (red) and Kronecker (green). These results provide compelling evidence that VQ achieves superior element-wise error control during the initialization period. This establishes a practical foundation for precision recovery under further quantization operations.
\begin{figure*}[h]
    \centering
    \begin{subfigure}[b]{0.48\textwidth}
        \centering
        \includegraphics[width=\linewidth, height=0.55\textwidth]{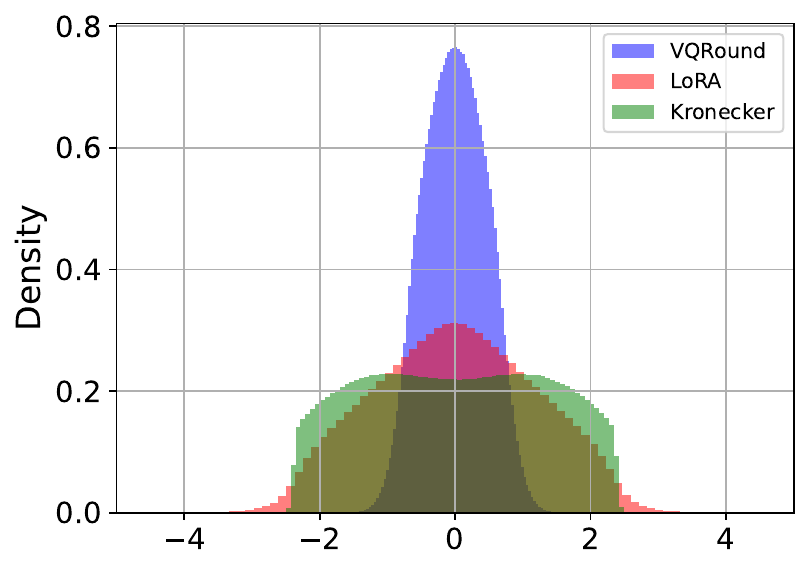}
        \caption{Error density distribution of $A - \hat{A}$.}
        \label{fig:dist_alpha}
    \end{subfigure}
    \hfill
    \begin{subfigure}[b]{0.48\textwidth}
        \centering
        \includegraphics[width=\linewidth, height=0.55\textwidth]{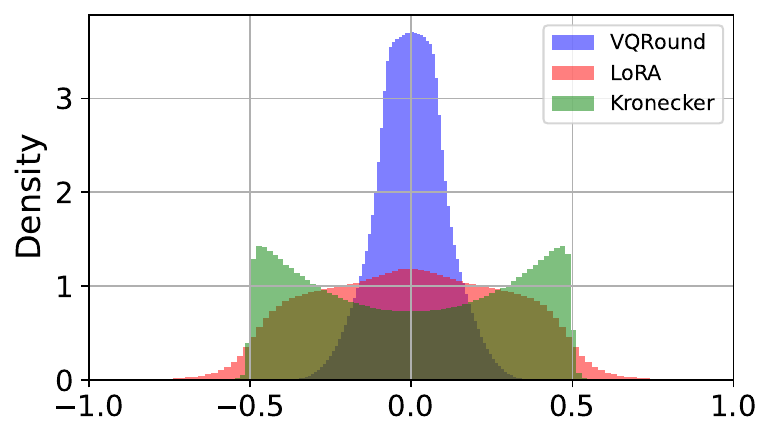}
        \caption{Error density distribution of $H - \hat{H}$.}
        \label{fig:dist_H}
    \end{subfigure}
    \caption{Comparison of error density distributions between VQRound, LoRA, and Kronecker on parameter matrix $A$ and rounding matrix $H$ after initialization. VQRound (blue) demonstrates a significantly higher density near zero, indicating superior element-wise error control.}
    \label{fig:distribution_comparison}
\end{figure*}

Fig.~\ref{fig:svd} shows the distribution of singular values across specific weight matrix on OPT-1.3B, LLaMA-2-7B, and LLaMA-3-8B. Across different model architectures, singular values maintain substantial magnitude ($> 10^0$) even at high ranks ($r=64$). These findings suggest that low-rank initialization methods like LoRA and Kronecker suffer significant information loss by discarding the higher-order information, thus explaining why low-rank approaches are less effective than VQ in adaptive rounding.
\begin{figure}[!htb]
\label{fig:svd}
\centering
\includegraphics[width=\linewidth, height=0.5\textwidth]{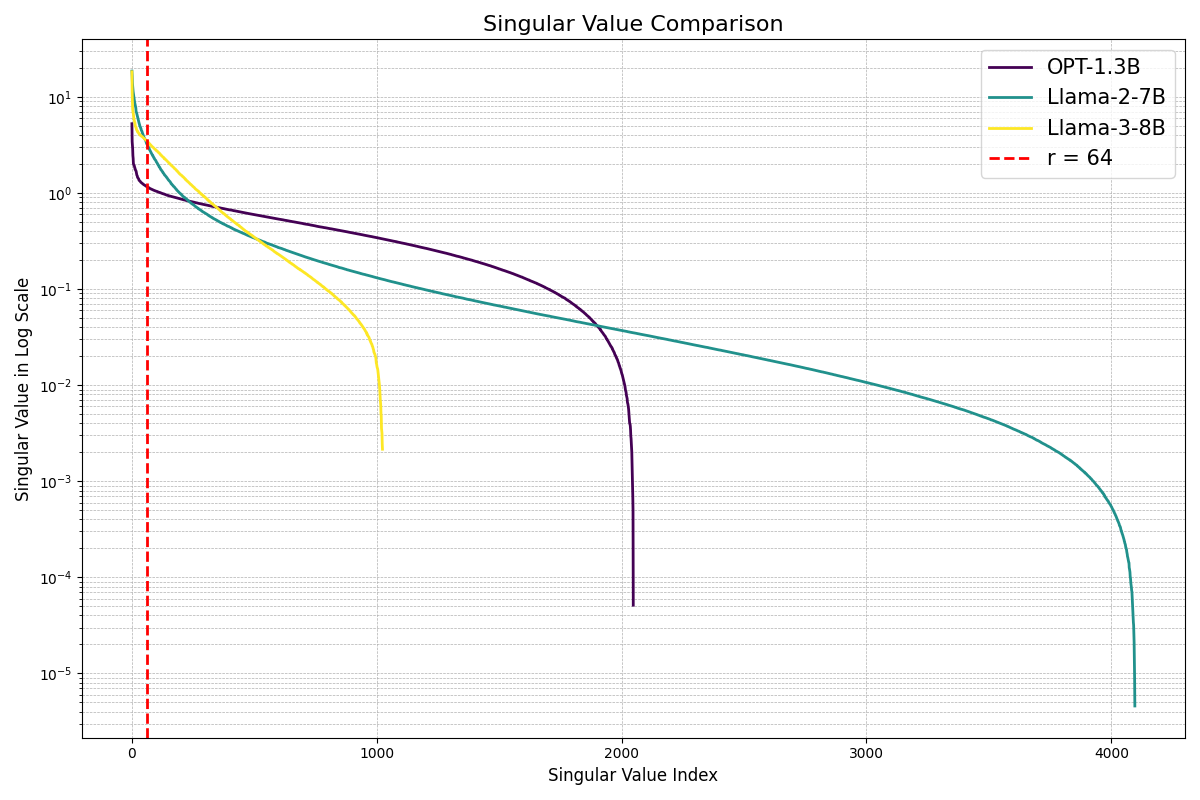} 
\captionof{figure}{Singular value distribution}
\label{fig:svd}
\end{figure}

\end{document}